%% file: main.tex
\documentclass{article} 
\usepackage{iclr2022_conference,times}

\input{math_commands.tex}

\input{more_math_commands.tex}

\usepackage{hyperref}
\usepackage{url}

\usepackage{cleveref}
\usepackage{booktabs}
\usepackage{amsthm}
\usepackage{amssymb}
\usepackage{algorithm}
\usepackage{algorithmic}
\usepackage{longtable}

\usepackage{tikz}
\usetikzlibrary{positioning}
\usetikzlibrary{patterns}
\usetikzlibrary{decorations.pathreplacing}

\usetikzlibrary{math} 

\tikzmath{\x1 = 0.5; \y1 =-0.2; 
}

\usepackage{listings}
\definecolor{codegreen}{rgb}{0,0.6,0}
\definecolor{codegray}{rgb}{0.5,0.5,0.5}
\definecolor{codepurple}{rgb}{0.58,0,0.82}
\definecolor{backcolour}{rgb}{0.95,0.95,0.92}

\lstdefinestyle{mystyle}{
    backgroundcolor=\color{backcolour},   
    commentstyle=\color{codegreen},
    keywordstyle=\color{magenta},
    numberstyle=\tiny\color{codegray},
    stringstyle=\color{codepurple},
    basicstyle=\ttfamily\footnotesize,
    breakatwhitespace=false,         
    breaklines=true,                 
    captionpos=b,                    
    keepspaces=true,                 
    numbersep=5pt,                  
    showspaces=false,                
    showstringspaces=false,
    showtabs=false,                  
    tabsize=2
}

\newcommand{\blap}[1]{\smash[b]{\begin{tabular}[t]{@{}c@{}}#1\end{tabular}}}

\newtheorem{theorem}{Theorem}[section]
\crefname{theorem}{Theorem}{Theorems}

\crefname{lemma}{Lemma}{Lemmas}

\crefname{corollary}{Corollary}{Corollaries}
\newtheorem{proposition}[theorem]{Proposition}
\crefname{proposition}{Proposition}{Propositions}

\crefname{conjecture}{Conjecture}{Conjecture}

\theoremstyle{definition}

\newtheorem{definition}[theorem]{Definition}
\crefname{definition}{Definition}{Definitions}
\crefname{figure}{Figure}{Figures}
\crefname{SCfigure}{Figure}{Figures}
\crefname{section}{Section}{Sections}
\crefname{subsection}{Section}{Sections}

\theoremstyle{remark}
\newtheorem{remark}[theorem]{Remark}

\definecolor{cbblue}{HTML}{88BEEC}
\definecolor{cbyellow}{HTML}{FFE9A5}

\iclrfinalcopy

\title{VC dimension of partially quantized neural networks in the overparametrized regime}

\author{Yutong Wang\textsuperscript{1} \& Clayton Scott\textsuperscript{1 2}
\\
\textsuperscript{1}Department of Electrical Engineering and Computer Science\\
\textsuperscript{2}Department of Statistics\\
University of Michigan\\
Ann Arbor, MI 48109, USA \\
\texttt{\{yutongw,clayscot\}@umich.edu} \\
}

\begin{document}

\maketitle

\begin{abstract}
Vapnik-Chervonenkis (VC) theory has so far been unable to explain the small generalization error of overparametrized neural networks.
Indeed, existing applications of VC theory to large networks obtain upper bounds on VC dimension that are proportional to the number of weights, and for a large class of networks, these upper bound are known to be tight.
  In this work, we focus on a class of partially quantized networks that we refer to as \emph{hyperplane arrangement neural networks} (HANNs).
Using a sample compression analysis, we show that HANNs can have VC dimension significantly smaller than the number of weights, while being highly expressive.
In particular, empirical risk minimization over HANNs in the overparametrized regime achieves the minimax rate for classification with Lipschitz posterior class probability.
We further demonstrate the expressivity of HANNs empirically.
On a panel of 121 UCI datasets, overparametrized HANNs
match the performance of state-of-the-art full-precision models.
\end{abstract}

\section{Introduction}

Neural networks have become an indispensable tool for machine learning practitioners, owing to their impressive performance especially in vision and natural language processing \citep{goodfellow2016deep}.
In practice, neural networks are often applied in the \emph{overparametrized} regime and are capable of fitting even random labels
\citep{zhang2021understanding}.
Evidently, these overparametrized models perform well on real world data despite their ability to grossly overfit, a phenomenon that has been dubbed ``the generalization puzzle'' \citep{nagarajan2019uniform}.

Toward solving this puzzle, several research directions have flourished and offer 
potential
explanations, including
implicit regularization 
\citep{chizat2020implicit},
interpolation
\citep{chatterji2021finite},
and
benign overfitting
\citep{bartlett2020benign}.
So far, VC theory has not been able to explain the puzzle,
because 
existing bounds on the VC dimensions of neural networks are
on the order of the number of weights \citep{maass1994neural,bartlett2019nearly}.
It remains unknown whether there exist neural network architectures capable of modeling rich set of classfiers with low VC dimension.

The focus of this work is on a class of neural networks with threshold activation that we refer to as \emph{hyperplane arrangement neural networks} (HANNs).
Using the theory of sample compression schemes \citep{littlestone1986relating}, we show that HANNs can have VC dimension that is significantly smaller than the number of parameters.
Furthermore, we apply this result to show that HANNs have high expressivity by proving that HANN classifiers achieve minimax-optimality when the data has Lipschitz posterior class probability in an overparametrized setting.

We benchmark the empirical performance of HANNs on a panel of 121 UCI datasets, following several recent neural network and neural tangent kernel works \citep{klambauer2017self,wu2018improved,arora2019harnessing,shankar2020neural}.
In particular, \cite{klambauer2017self} showed that, using a properly chosen activation, overparametrized neural networks perform competitively compared to classical shallow methods on this panel of datasets.
Our experiments show that HANNs, a partially-quantized model, match the classification accuracy of the self-normalizing neural network \citep{klambauer2017self} and the dendritic neural network \citep{wu2018improved}, both of which are full-precision models.

\subsection{Related work}

\textbf{VC dimensions of neural networks.}
The 
\emph{VC-dimension}
\cite{vapnik1971uniform}
is a combinatorial measure of the complexity of a concept class, i.e., a set of classifiers.
The \emph{Fundamental Theorem of Statistical Learning} \cite[Theorem 6.8]{shalev2014understanding} states that 
a concept class has finite VC-dimension if and only if it is probably approximately correct (PAC) learnable, where
 the VC-dimension is tightly related to the number of samples required for PAC learning.

For threshold networks, \cite{cover1965geometrical,baum1989size} showed a VC-dimension upper bounded of $O(w \log w)$, where $w$ is the number of parameters.
\cite{maass1994neural} obtained a matching \emph{lower} bound attained by a network architecture with two hidden layers.
More recently, \cite{bartlett2019nearly} obtained the upper and lower bounds $O(w \ell  \log w)$ and $\Omega(w \ell \log (w/\ell))$ respectively for the case when the activation is piecewise linear, where $\ell$ is the number of layers.
These lower bounds are achieved by somewhat unconventional network architectures.
The architectures we consider exclude these and thus we are able to achieve a smaller upper bound on the VC dimensions.

\textbf{The generalization puzzle.}
In practice, neural networks that achieve state-of-the-art performance use significantly more parameters than samples, a phenomenon that cannot be explained by classical VC theory if the VC dimension $\ge$ number of weights. This has been dubbed the \emph{generalization puzzle} \citep{zhang2021understanding}.
To explain the puzzle, researchers have pursued new directions including
margin-based bounds 
\citep{neyshabur2017exploring,bartlett2017spectrally},
PAC-Bayes bounds
\citep{dziugaite2017computing},
and
implicit bias of optimization methods
\citep{gunasekar2018implicit,chizat2020implicit}.
We refer the reader to the recent article by \cite{bartlett2021deep} for a comprehensive coverage of this growing literature.

The generalization puzzle is \emph{not} specific to deep learning.
For instance, AdaBoost has been observed to continue to decrease the test error while the VC dimension grows linearly with the number of boosting rounds \citep{schapire2013explaining}.
Other learning algorithms that exhibits similarly surprising behavior include random forests \citep{wyner2017explaining} 
and kernel methods \citep{belkin2018understand}.

\textbf{Minimax-optimality.}
Whereas VC theory is distribution-independent,
minimax theory is concerned with the question of \emph{optimal} estimation/classification under distributional assumptions\footnote{Without distributional assumptions, no classifier can be minimax optimal in light of the No-Free-Lunch Theorem \citep{devroye1982any}.}.
A minimax optimality result shows that the expected excess classification error goes to zero at the fastest rate possible, as the sample size tend to infinity.
For neural networks, this often involves a hyperparameter selection scheme in terms of the sample size.

\cite{farago1993strong} show minimax-optimality of (underparametrized) neural networks for learning to classify under certain assumptions on the Fourier transform of the data distribution.
\cite{schmidt2020nonparametric} shows minimax-optimality of $s$-\emph{sparse} neural networks for regression over H\"older classes, where at most $s = O(n \log n)$ network weights are nonzero, and $n =$ the number of training samples.
\cite{kim2021fast}
extends the results of \cite{schmidt2020nonparametric} to the classification setting, remarking that effective optimization under sparsity constraint is lacking.
\cite{kohler2020rate} and \cite{langer2021analysis} proved minimax-optimality without the sparsity assumption, however in an underparametrized setting.
To the best of our knowledge, our result is the first to establish minimax optimality of overparametrized neural networks without a sparsity assumption.

\textbf{(Partially) quantized neural networks.}
Quantizing some of the weights and/or activations of neural networks has the potential to reduce the high computational burden of neural networks at test time \citep{qin2020binary}.
Many works have focused on the efficient training of quantized neural networks to close the performance gap with full-precision architectures \citep{hubara2017quantized,rastegari2016xnor,lin2017towards}.
Several works have observed that quantization of the activations, rather than of the weights, leads to a larger accuracy gap
\citep{cai2017deep,mishra2018wrpn,kim2019binaryduo}.

Towards explaining this phenomenon, researchers have focused on understanding the so-called \emph{coarse gradient}, a term coined by \cite{yin2019understanding}, often used in training QNNs as a surrogate for the usual gradient.
One commonly used heuristic is the \emph{straight-through-estimator} (STE) first introduced in an online course by \cite{hinton2012neural}.
Theory supporting the STE heuristic has recently been studied in
\cite{li2017training}
 and
\cite{yin2019understanding}.

QNNs have also been analyzed from other theoretical angles, including mean-field theory
\citep{blumenfeld2019mean},
memory capacity \citep{vershynin2020memory},
Boolean function representation capacity  \citep{baldi2019capacity}
and adversarial robustness 
\citep{lin2018defensive}.
Of particular relevance to our work, \cite{maass1994neural} constructed an example of a QNN architecture with VC dimension on the order of the number of weights in the network. In contrast, our work shows that there exist QNN architectures with much smaller VC dimensions.

\textbf{Sample compression schemes.}
Many concept classes with geometrically structured decision regions, such as axis-parallel rectangles, can be trained on a properly chosen size $\sigma$ subset of an arbitrarily large training dataset without affecting the result.
Such a concept class is said to admit a \emph{sample compression schemes} of size $\sigma$, a notion introduced by \cite{littlestone1986relating} who showed that the VC dimension of the class is upper bounded by $O(\sigma)$.
Furthermore, the authors posed the \emph{Sample Compression Conjecture}.
See \cite{moran2016sample} for the best known partial result and an extensive review of research in this area.
Besides the conjecture, sample compression schemes have also been applied to other long-standing problems in learning theory
\citep{hanneke2019sample,bousquet2020proper,ashtiani2020near}.
To the best of our knowledge, our work is the first to apply sample compression schemes to neural networks.

\section{Notations}

The set of real numbers is denoted $\mathbb{R}$.
The unit interval is denoted $[0,1]$.
For an integer $k \ge 1$, let $[k] = \{1,\dots, k\}$.
We use $\mathcal{X}$ to denote the feature space, which in this work will either be $\mathbb{R}^d$ or $[0,1]^d$ where $d \ge 1$ is the ambient dimension/number of features.

Denote by
$\mathbb{I}\{\mathtt{input}\}$ the \emph{indicator} function which returns $1$ if $\mathtt{input}$ is true and $0$ otherwise.
The \emph{sign} function is given by $\sgn(t) = \mathbb{I}\{t \ge 0\}  - \mathbb{I}\{t < 0\}$.
For vector inputs, $\sgn$ applies entry-wise.

The set of labels for binary classification is denoted $\mathbb{B} := \{\pm 1\}$.
Joint distributions on $\mathcal{X} \times \mathbb{B}$ are denoted by $P$, where 
$X,Y \sim P$ denotes a random instance-label pair distributed according to $P$.
Let $f: \mathcal{X} \to \mathbb{B}$ be a binary classifier. The \emph{risk} with respect to $P$ is denoted by $R_P(f) := P( f(X) \ne Y)$.
For an integer $n \ge 1$, the \emph{empirical risk} is the random variable $\hat R_{P,n}(f) := \frac{1}{n} \sum_{i=1}^n \mathbb{I}\{ f(X_i) \ne Y_i\}$, where $(X_1, Y_1),\dots, (X_n,Y_n) \sim P$ are i.i.d.
The \emph{Bayes risk} $\inf_{f: \mathcal{X} \to \mathbb{B}} R_P(f)$ with respect to $P$ is denoted by $R_P^*$.

Let $f,g : \{1,2,\dots\} \to \mathbb{R}_{\ge 0}$ be nonnegative functions on the natural numbers. We write $f \asymp g$ if there exists $\alpha,\beta >0$ such that 
for all $n = 1,2,\dots$ we have $\alpha g(n) \le f(n) \le \beta g(n)$.

\section{Hyperplane arrangement neural networks}
\label{gen_inst}

A hyperplane $H$ in $\mathbb{R}^d$ is specified by its normal vector $\bw \in \mathbb{R}^d$ and bias $b \in \mathbb{R}$.
The mapping $x \mapsto \sgn(\bw^\top x +b)$ indicates the side of $H$ that $x$ lies on, and hence induces a partition of $\mathbb{R}^d$ into two halfspaces.
A set of $k \ge 1$ hyperplanes is referred to as a $k$-\emph{hyperplane arrangement}, and specified by a matrix of normal vectors 
  and a vector of offsets:
\[
\bW = 
[
  \bw_1 \cdots \bw_k
  ] \in \mathbb{R}^{d \times k}
  \quad \mbox{and}\quad
\biasb = [b_1,\dots, b_k]^\top.
\]
Let
  $
    q_{\bW, \biasb}(x) := \sgn(\bW^\top x + \biasb)$ for all $x \in \mathbb{R}^d.
  $
  The vector $q_{\bW,\biasb}(x) \in \mathbb{B}^k$ is called a \emph{sign vector} and the set of all realizable sign vectors
  is denoted
  $\mathfrak{S}_{\bW,\biasb} := \{q_{\bW,\biasb}(x) : x \in \mathbb{R}^d\}.$
  Each sign vector $\bs \in \mathfrak{S}_{\bW,\biasb}$ uniquely defines a set $\{x \in \mathbb{R}^d: q_{\bW, \biasb}(x) = \bs\}$ known as a \emph{cell} of the hyperplane arrangement.
  The set of all cells forms a partition of $\mathbb{R}^d$.
For an example, see \cref{fig: hyperplane arrangement}-left.

A classical result in the theory of hyperplane arrangement due to \cite{buck1943partition} gives the following tight upper bound on the number of distinct sign patterns/cells:
\begin{equation}
  \label{equation: upper bound on the number of cells}
  |\mathfrak{S}_{\bW,\biasb}| \le \binom{k}{\le d}
  := 
  \begin{cases}
    2^k &: k < d,\\
    \binom{k}{0} + \binom{k}{1} + \cdots + \binom{k}{d} &: k \ge d.
  \end{cases}
\end{equation}
See \cite{fukuda2013lecture} Theorem 10.1 for a simple proof.
A \emph{hyperplane arrangement classifier} assigns a binary label $y \in \mathbb{B}$ to a point $x \in \mathbb{R}^d$ solely based on the sign vector $q_{\bW,\biasb}(x)$.
\begin{figure}[htpb]
  \centering
\begin{tikzpicture}[scale=0.9,>=stealth]
        \fill[gray!15] (1, 2) -- (0.2,1.2)-- (10/18,-5/18) --(2,-1) -- (2,2) -- cycle;
 \draw[ultra thick,yshift=0.5] (-2,-1) -- (1,2) ;
 \draw[ultra thick, lightgray, yshift=-1.5] (-2,-1) -- (1,2) ;
 \draw[ultra thick,xshift=0.5]  (0,2) -- (1,-2) ;
 \draw[ultra thick, lightgray, xshift=-1]  (0,2) -- (1,-2) ;
 \draw[ultra thick] (-2,1)--(2,-1);
 \draw[ultra thick, lightgray, yshift=1.5] (-2,1)--(2,-1);
 \node[above right] at (-2,1) {$H_1$};
 \node[below right] at (-2,-1) {$H_2$};
 \node[above left] at (1,-2) {$H_3$};
 \node at (-0.05,0.45) {\texttt{+++}};
 \node at (-0.5,-0.5) {\texttt{-++}};
 \node at (-1.5,0) {\texttt{--+}};
 \node at (-0.8,1.5) {\texttt{+-+}};
 \node (A) at (1.3,0.5) {{\texttt{++-}}};
 \draw [->,thick] (A) to [bend left=13] (2.8,.85);
 \node at (0.44,1.86) {\texttt{+--}};
 \node at (1.3,-1.3) {\texttt{-+-}};
        \draw[thick] (-2,-2)--(2,-2)--(2,2)--(-2,2)--cycle;
        \node[circle,fill=black,scale=0.6] at  (-2/3,1/3) {};
        \node[circle,fill=black,scale=0.6,xshift=1.5,yshift=-0.5] at  (10/18,-5/18) {};
        \node[circle,fill=black,scale=0.6,xshift=0.7,yshift=1.0] at (0.2,1.2) {};
\end{tikzpicture}
\begin{tikzpicture}
  \fill[gray!20] (-1,0.9) rectangle (0.15,0.5);
  \draw[thick] (-1,0.9) rectangle (0.15,0.5);
  \node  at (0,0) {
      $\begin{array}{l|l}
        \mathtt{B_1B_2B_3} & \,\,\mathtt{Y}\\
        \hline
        \mathtt{+++} & \colorbox{cbblue}{$\mathtt{+}$}\\
        {\mathtt{++-}} & \colorbox{cbblue}{$\mathtt{+}$}\\
        \mathtt{+-+} & \colorbox{cbyellow}{$\mathtt{-}$}\\
        \mathtt{+--} & \colorbox{cbblue}{$\mathtt{+}$}\\
        \mathtt{-++} & \colorbox{cbyellow}{$\mathtt{-}$}\\
        \mathtt{-+-} & \colorbox{cbblue}{$\mathtt{+}$}\\
        \mathtt{--+} & \colorbox{cbyellow}{$\mathtt{-}$}
      \end{array}$
    };
  \draw[thick] (0.44,0.92) rectangle (0.91,0.46);
\end{tikzpicture}
\begin{tikzpicture}[scale=0.9]
  \fill[cbyellow] (-2,-2)--(2,-2)--(2,2)--(-2,2)--cycle;
        \fill[cbblue] (0, 2) -- (0.2,1.2)-- (-2/3,1/3)-- (10/18,-5/18) --(1,-2)--(2,-2) -- (2,2) -- cycle;
 \draw[dotted,thick] (-2,-1) -- (1,2) ;
 \draw[dotted, thick]  (0,2) -- (1,-2) ;
 \draw[dotted,thick] (-2,1)--(2,-1);
 \node at (-0.05,0.45) {\texttt{+}};
 \node at (-0.5,-0.5) {\texttt{-}};
 \node at (-1.5,0) {\texttt{-}};
 \node at (-0.8,1.5) {\texttt{-}};
 \node (A) at (1.3,0.5) {{\texttt{+}}};
 \node at (0.35,1.7) {\texttt{+}};
 \node at (1.3,-1.3) {\texttt{+}};
 \draw [<-,thick] (A) to [bend right=13] (-2.8,.85);
 \draw[thick] (-2,-2)--(2,-2)--(2,2)--(-2,2)--cycle;
\end{tikzpicture}
\caption{Left: An arrangement of $3$ hyperplanes $\{H_1,H_2,H_3\}$ in $\mathbb{R}^2$. There are $7$ sign patterns. Middle: An example of a lookup table (see Remark~\ref{remark: lookup table}). Right: the resulting classifier.}%
  \label{fig: hyperplane arrangement}
\end{figure}
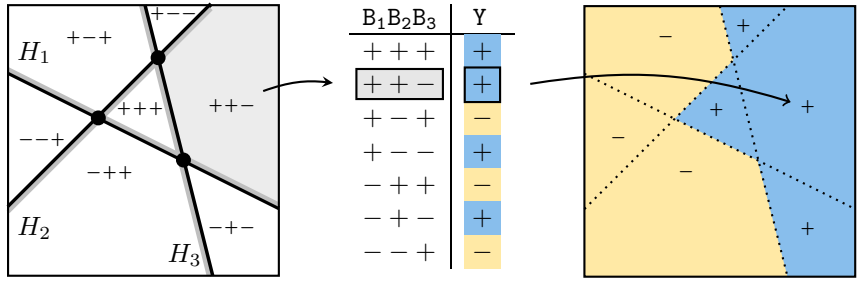

\begin{definition}
  \label{definition: HAC}
  Let $\mathbb{B}^{\mathcal{X}}$ be the set of all functions from $\mathcal{X}$ to $\mathbb{B}$.
  A \emph{concept class} $\mathcal{C}$ over $\mathcal{X}$ is a subset of $\mathbb{B}^{\mathcal{X}}$.
  Fix $r,k$ positive integers, $r \le \min \{d,k\}$.
  Let $\Bool{k}$ be the set of all Boolean functions $\mathbb{B}^k \to \mathbb{B}$.
  The \emph{hyperplane arrangement classifier} class is the concept class, denoted  $\HAC$, over $\mathbb{R}^d$ defined by
  \begin{align*}
    \HAC
    =
    \{ h \circ q_{\bW,\biasb} :\, \, & h \in \Bool{k},  \, q_{\bW,\biasb}(x) := \sgn(\bW^\top x + \biasb),\\
                                         &\bW \in \mathbb{R}^{d \times k},\,\mathrm{rank}(\bW) \le r,\, \biasb \in \mathbb{R}^k \}.
  \end{align*}
\end{definition}

See \cref{fig: HAC} for a graphical representation of $\HAC$.
When the set of Boolean functions is realized by a neural network, we refer to the resulting  classifier as a \emph{hyperplane arrangement neural network} (HANN).

\begin{remark}
  \label{remark: lookup table}
  Consider a fixed hyperplane arrangement $\bW$, $\biasb$ and Boolean function $h \in \Bool{k}$.
  When performing prediction with the classifer $h \circ q_{\bW, \biasb}$,
  the feature vector $x$ is mapped to a sign vector to which $h$ is applied.
  Thus, we do not need to know how $h$ behaves outside of $\mathfrak{S}_{\bW,\biasb}$.
  The restriction of $h$ to $\mathfrak{S}_{\bW,\biasb}$ is a \emph{partially defined Boolean function} or a \emph{lookup table}.
\end{remark}

\input{diagram.tex}

\begin{remark}
  \label{remark: latent subspace}
  The hidden layer of width $r$ in \cref{fig: HAC} allows the user to impose the restriction that the hyperplane arrangement classifier depends only on $r$ relevant features, which can be either learned or defined by data preprocessing.
  When $r= d$, no restriction is imposed.
  In this case, the input layer is directly connected to the Boolean layer.
  This is consistent with Definition~\ref{definition: HAC} where the rank constraint $\mathrm{rank}(\bW) \le r$ becomes trivial.
\end{remark}

Our next goal is to upper bound the VC dimension of $\HAC$.

\begin{definition}
  [VC-dimension]
  Let $\mathcal{C} \subseteq \mathbb{B}^{\mathcal{X}}$ be a concept class over $\mathcal{X}$.
  A set $S:=\{x_1,\dots, x_n\} \subseteq \mathcal{X}$ is \emph{shattered} by $\mathcal{C}$ if for all sequences
  $(y_1,\dots, y_n) \in \mathbb{B}^n$, there exists $f \in \mathcal{C}$ such that $f(x_i) = y_i$ for all $i \in [n]$.
  The \emph{VC-dimension} of $\mathcal{C}$ is defined as 
  \[
    \mathtt{VC}(\mathcal{C}) = \sup \{|S| : S \subseteq \mathcal{X}, \, 
      \mbox{$S$ is shattered by $\mathcal{C}$}
    \}.
  \]
\end{definition}

The VC-dimension has many far-reaching consequences in learning theory and, in particular, classification.
One of these consequences is a sufficient (in fact also necessary) condition for \emph{uniform convergence} in the sense of the following well-known theorem.
See \cite{shalev2014understanding} Theorem 6.8.
\begin{theorem}
  \label{theorem: uniform deviation bound}
  Let $\mathcal{C}$ be a concept class over $\mathcal{X}$.
  There exists a constant $C>0$ such that for all joint distributions $P$ on $\mathcal{X} \times \mathbb{B}$ and all $f \in \mathcal{C}$,  we have
$
  |\hat {R}_{P,n}(f) - R_P(f)| \le 
  C
  \sqrt{
  ( \mathtt{VC}(\mathcal{C}) + \log(1/\delta))/{n}
  }
$
with probability at least $1-\delta$ with respect to the draw of $(X_1,Y_1),\dots, (X_n,Y_n)$.
\end{theorem}
Note that the above VC bound is useless in the overparametrized setting if $\mathtt{VC}(\mathcal{C}) = \Omega(\mbox{\# of weights})$.
We now present our main result: an upper bound on the VC dimension of $\HAC$.
\begin{theorem}
  \label{theorem: VC dimension bound for HAC}
  Let $d,r,k \ge 1$ be integers and $\HAC$ be defined as in Definition~\ref{definition: HAC}. Then
  \[\mathtt{VC}(\HAC)
    \le 8\cdot\left(
    k(d+1) + k(d+1)(1+ \lceil \log_2 k \rceil)  + \binom{k}{\le r}
    \right).
  \]
\end{theorem}
In the next section, we will prove this result using a sample compression scheme.
Before proceeding, we comment on the significance of the result.

\begin{remark}
  \label{remark: Boolean implementations}
Since $\binom{k}{\le r} = O(k^r)$, we have 
$
  \mathtt{VC}(\HAC) = 
  O(k^r + dk \log k)
  $ which only involves the input dimension $d$ and the width of the first two hidden layers $r$ and $k$.
  For constant $d$ and $r \ge 2$, 
  this reduces to $
  \mathtt{VC}(\HAC) = 
  O(k^r)$.
  In particular, the number of weights used by an architecture to implement the Boolean function $h$ does not affect the VC dimension at all.

For instance, \cite{mukherjee2017lower} Lemma 2.1 states that a 1-hidden layer neural network with ReLU activation can model any $k$-input Boolean function if the hidden layer has width $\ge 2^k$.
Note that this network uses $\ge k2^k$ weights, and $k2^k \gg k^r$ for fixed $r$ and $k$ large.

\cite{baldi2019capacity} study implementation of Boolean functions using threshold networks.
A consequence of their Theorem 9.3 is that a 2-hidden layer network with  widths $\ge c2^{k/2}/\sqrt{k}$ can implement all $k$ input Boolean functions, where $c$ is a constant not depending on $k$.
This requires $\ge  c^22^k/k$ weights which again is exponentially larger than $k^r$.
Furthermore, this lower bound on the weights is also necessary as $k\to \infty$.
\end{remark}

\section{A sample compression scheme}

  In this section, we will construct a sample compression scheme for $\HAC$.
As alluded to in the Related Work section, the size of a sample compression scheme upper bounds the VC-dimension of a concept class, which will be applied to prove 
  \cref{theorem: VC dimension bound for HAC}.
  We first recall the definition of  sample compression schemes with side information introduced in \cite{littlestone1986relating}. 

\begin{definition}
  Let $\mathcal{C}$ be a concept class.
  A length $n$ sequence $\{(x_i,y_i) \in \mathcal{X} \times \mathbb{B}\}_{i \in [n]}$
  is $\mathcal{C}$-\emph{labelled} if there exists $f \in \mathcal{C}$
  such that $f(x_i) = y_i$ for all $i \in [n]$.
  Denote by $L_{\mathcal{C}}(n)$ the set of $\mathcal{C}$-labelled sequences of length at most $n$.
  Denote by $L_{\mathcal{C}}(\infty)$ the set of all $\mathcal{C}$-labelled sequences of finite length.
  The concept class $\mathcal{C}$ over $\mathcal{X}$ has an
  \emph{$m$-sample compression scheme with $s$-bits of side information} 
  if there exists a pair of maps $(\rho, \kappa)$ where
  \[
    \kappa : L_{\mathcal{C}}(\infty) \to L_{\mathcal{C}}(m) \times \mathbb{B}^s,
    \quad
    \rho : L_{\mathcal{C}}(m) \times \mathbb{B}^s
    \to 
    \mathbb{B}^{\mathcal{X}}
  \]
  such that for all $\mathcal{C}$-labelled sequences $S := \{(x_i,y_i)\}_{i \in [n]}$, we have
  $\rho(\kappa(S))(x_i) = y_i$ for all $i \in [n]$.
  The \emph{size} of the sample compression scheme is $ \mathtt{size}(\rho,\kappa) := m + s$.
\end{definition}

Intuitively, $\kappa$ and $\rho$ can be thought of as the \emph{compression} and the \emph{reconstruction} maps, respectively.
The compression map $\kappa$ keeps $m$ elements from the training set and $s$ bits of additional information, which 
$\rho$ uses to reconstruct a classifier that correctly labels the uncompressed training set.

The main result of this section is:
\begin{theorem}
  \label{theorem: sample compression scheme}
  $\HAC$ has a sample compression scheme $(\rho,\kappa)$ of size 
  \[
    \mathtt{size}(\rho,\kappa)=
    k(d+1) + k(d+1)(1+ \lceil \log_2 k \rceil)  + \binom{k}{\le r}.
  \]
\end{theorem}
The rest of this section will work toward the proof of \cref{theorem: sample compression scheme}.
The following result states that a $\mathcal{C}$-labelled sequence can be labelled  by a hyperplane arrangement classifier of a special form.
\begin{proposition}
  \label{proposition: canonical form}
Let $\{(x_i,y_i)\}_{i\in[n]}$ be $\HAC$-labelled. Then there exist $\bV
=[
  \bv_1 \cdots \bv_k
  ] \in \mathbb{R}^{d \times k}, \biasc \in \mathbb{R}^k$ and $h \in \Bool{k}$ such that for all $i \in [n]$, we have
      1)
      $y_i = h(\sgn(\bV^\top x_i + \biasc))$,
  2)
  $\mathrm{rank}(\bV) \le r$ and 3)
  $| \bv_j^\top x_i + c_j| \ge 1$ for all $i \in [n],j \in [k]$.
\end{proposition}
The proof, given in Appendix~\ref{section: canonical form}, is similar to showing the existence of a max-margin separating hyperplane for a linearly separable dataset.

\begin{definition}
  Let $I$ be a finite set and let $\ba_i \in \mathbb{R}^n$ for each $i \in I$.
  Let $A = \{\ba_i\}_{i \in I}$.
  A \emph{conical combination} of $A$ is
  a linear combination $\sum_{i \in I} \lambda_i \ba_i$ where the weights $\lambda_i \in \mathbb{R}_{\ge 0}$ are nonnegative.
  The \emph{conical hull} of $A$, denoted $\mathtt{coni}(A)$, is the set of all conical combinations of $A$, i.e.,
  $
    \mathtt{coni}(\{\ba_i\}_{i \in I})
    :=
    \left\{ \sum_{i \in I} \lambda_i \ba_i : \lambda_i \in \mathbb{R}_{\ge 0} ,\,\forall i \in I\right\}.
  $
\end{definition}

The result below follows easily from 
  the Carath\'edory's theorem for the conical hull \citep{lovasz2009matching}.
  For the sake of completeness, we included the proof in Appendix~\ref{section: proof of caratheodory theorem}.
\begin{proposition}
  \label{proposition: caratheodory theorem}
  Let $\ba_1,\dots, \ba_m \in \mathbb{R}^n$ and $b_1,\dots, b_m \in \mathbb{R}$.
  For each subset $I \subseteq [m]$, define
  \[\mathcal{P}_I := \{ x \in \mathbb{R}^n : \ba_i^\top x \le b_i \, \forall i \in I\}.\]
  Then 1)
      $\min_{x \in \mathcal{P}_I} \frac{1}{2} \|x\|^2$ has a unique minimizer, denoted by $x^*_I$ below, and
      2)
  there exists a subset $J \subseteq [m]$ such that 
  $|J| = n$ and for all $I \subseteq [m]$ with $J \subseteq I$, we have $x^*_{[m]} = x^*_I$.
\end{proposition}

\begin{proof}[Proof of \cref{theorem: sample compression scheme}]
  Let $(x_i,y_i)$ be $\HAC$-realizable, and $\bV,\biasc$ and $h$ be as in Proposition~\ref{proposition: canonical form}.
  For each $i \in [n]$, define the Boolean vectors $\bs_i := \sgn(\bV^\top x_i + \biasc) \in \{\pm 1\}^k$ and $s_{ij} = \sgn(\bv_j^\top x_i + c_j)$ denote the $j$-th entry of $\bs_i$.
  Note that
  $
    s_{ij} ( \bv_j^\top x_i + c_j)= 
    | \bv_j^\top x_i + c_j| \ge 1
  $.

  We first outline the steps of the proof:
  \begin{enumerate}
    \item 
      Using a subset of the samples $\{(x_{i_\ell}, y_{i_\ell}) : \ell \in [d(k+1)]\}$
      with additional $k(d+1)(1 + \lceil \log_2 k \rceil)$ bits of side information $\{(s_{i_\ell j_\ell}, j_{\ell}) : \ell \in [d(k+1)]\}$, we can reconstruct
      $\overline{\bW}, \overline{\biasb}$
such that 
$\sgn(\overline{\bW}^\top x_i + \overline{\biasb}) = \bs_i$ for all $i \in [n]$.
\item Using an additional subset of samples $\{(x_{\iota_\ell}, y_{\iota_\ell}) : \ell =1,\dots, \binom{k}{\le r}\}$ in conjunction with the
  $\overline{\bW}, \overline{\biasb}$ reconstructed in the previous step, we can find $g \in \Bool{k}$ such that $g(s_i) = h(s_i)$ for all $i$.
  \end{enumerate}
  

  Now, consider the set
  \[
    \mathcal{P}:=
    \left\{ (\bW, \biasb) \in \mathbb{R}^{d \times k} \times \mathbb{R}^k : 
    s_{ij}  (\bw_j^\top x_i + b_j )\ge 1,\, \forall i \in [n],\, j \in [k]\right\}.
  \]
  Note that $\mathcal{P}$ is a convex polyhedron in $(d + 1)k$-dimensional space.
  Let $(\overline{\bW}, \overline{\biasb})$ be the minimum norm element of $\mathcal{P}$.
  Note that $\sgn(\overline{\bW}^\top x_i + \overline{\biasb}) = \sgn(\bV^\top x_i + \biasc) = \bs_i$ by construction.

  By Proposition~\ref{proposition: caratheodory theorem}, there exists a set of tuples
  \[
    \left\{(i_\ell,j_\ell)\right\}_{\ell = 1,\dots, (d+1)k},
    \mbox{ where }
    (i_\ell,j_\ell) \in [n]\times [k]
  \]
  such that $\overline{\bW}, \overline{\biasb}$ is also the minimum norm element of 
  \[
    \mathcal{P}':=
    \left\{
      (\bW, \biasb) \in \mathbb{R}^{d \times k} \times \mathbb{R}^k : 
    s_{i_\ell j_\ell}  (\bw_{j_\ell}^\top x_{i_\ell} + {b}_{j_\ell}) \ge 1,\, \ell = 1,\dots, d(k+1) \right\}.
  \]
  
  To encode the defining equations of $\mathcal{P}'$, we need to store
  \begin{equation}
    \label{equation: samples for reconstructioning quantizer}
    \mbox{samples $\{(x_{i_\ell}, y_{i_\ell})\}_{\ell=1}^{d(k+1)}$
  and side information $\{(s_{i_\ell j_\ell}, j_{\ell})\}_{\ell=1}^{d(k+1)}$.
}
\end{equation}
  Note that each $s_{i_\ell j_\ell}$ requires $1$ bit while each $j_\ell \in [k]$ requires $\lceil \log_2 k \rceil$ bits.
  In total, encoding $\mathcal{P}'$ requires storing $d(k+1)$ samples and $d(k+1)(1 + \lceil \log_2 k\rceil)$ of bits.

  To reconstruct $g \in \Bool{k}$ that agrees with $h$ on all the samples, it suffices to know $h$ when restricted to $\{\bs_i\}_{i=1}^n$.
  Since $\{\bs_i\}_{i=1}^n$ is a subset of $\mathfrak{S}_{\overline{\bW}, \overline{\biasb}}$, we have by \cref{equation: upper bound on the number of cells} that $|\{\bs_i\}_{i}^n| \le \binom{k}{\le r}$.
  Thus, $\{\bs_i\}_{i=1}^n$ has at most $\binom{k}{\le r}$ unique elements.
  Let
  $\left\{\bs_{{\iota}_\ell} : \ell = 1,\dots, \binom{k}{\le r}\right\}$
  be a set containing all such unique elements.
  Thus, we store 
  \begin{equation}
    \label{equation: samples for reconstructing the Boolean function}
    \mbox{samples $\{(x_{\iota_\ell},y_{\iota_\ell}) : \ell = 1,\dots, \binom{k}{\le r}\}$}.
  \end{equation}
  Using $\overline{\bW}, \overline{\biasb}$ as defined above, we have $\bs_{\iota_\ell} = \sgn(\overline{\bW}^\top x_{\iota_\ell} + \overline{\biasb})$.
  Now, simply choose $g$ such that $g(\bs_{\iota_\ell}) = y_{\iota_\ell}$ for all $\ell = 1,\dots, \binom{k}{\le r}$.

  To summarize, we formally define the compression  and reconstruction functions $(\kappa,\rho)$.
  Let $\kappa$ take the full sample $\{(x_i,y_i)\}_{i=1}^n$ and output the subsample (and side information) in \cref{equation: samples for reconstructioning quantizer} and \cref{equation: samples for reconstructing the Boolean function}.
  The reconstruction function $\rho$ first constructs $\overline{\bW}, \overline{\biasb}$ using \cref{equation: samples for reconstructioning quantizer}.
  Next, $\rho$ constructs $g$ using $\overline{\bW}, \overline{\biasb}$ and the samples of \cref{equation: samples for reconstructing the Boolean function}.
\end{proof}

Now, the following result\footnote{
See also \cite{naslund2017compression} Theorem 2 for a succinct proof.}
together with the sample compression scheme for $\HAC$ we constructed  imply \cref{theorem: VC dimension bound for HAC} from the previous section.
\begin{theorem}
  [\cite{littlestone1986relating}]
  If $\mathcal{C}$ has sample compression scheme $(\rho,\kappa)$, then $\mathtt{VC}(\mathcal{C}) \le 8\cdot \mathtt{size}(\rho,\kappa)$.
\end{theorem}

\begin{remark}
  Note that the reconstruction function $\rho$ is \emph{not} permutation-invariant.
  Furthermore, the overall sample compression scheme $\rho,\kappa$ is \emph{not} stable in the sense of \cite{hanneke2021stable}.
  In general, sample compression schemes with permutation-invariant $\rho$ 
 \citep{floyd1995sample}
 and 
 \emph{stable} sample compression schemes
  \citep{hanneke2021stable}
 enjoy tighter generalization bounds compared to ordinary sample compression schemes.
 We leave as an open question whether $\HAC$ has such specialized compression schemes.
\end{remark}

\section{Minimax-optimality for learning Lipschitz class}

In this section, we show that empirical risk minimization (ERM) over $\HAC$, for properly chosen $r$ and $k$, is minimax optimal for classification where the posterior class probability function is $L$-Lipschitz, for fixed $L >0$.
Furthermore, the choices for $r$ and $k$ is such that the associated HANN, the neural network realization of $\HAC$, is overparametrized for 
the Boolean function implementations discussed in Remark~\ref{remark: Boolean implementations}.

Below, let $X \in [0,1]^d$ and $Y \in \mathbb{B}$ be the random variables corresponding to a sample and label jointly distributed according to $P$.
Write $\eta_P(x) := P(Y = 1| X = x)$ for the posterior class probability function.

Let $\Sigma(L,[0,1]^d)$ denote the class of $L$-Lipschitz functions $f: [0,1]^d \to \mathbb{R}$, i.e.,
\begin{equation*}
  |f(x) - f(x')| \le L \| x - x'\|_2, \quad \forall x,x' \in [0,1]^d.
\end{equation*}
The following minimax lower bound result\footnote{
  The result we cite here is a special case of \cite[Theorem 3.5]{audibert2007fast}, which gives minimax lower bound for when $\eta_P$ has additional smoothness assumptions.
}
concerns classification when $\eta_P$ is $L$-Lipschitz:
\begin{theorem}
  [\cite{audibert2007fast}]
  There exists a constant $C>0$ such that
  \begin{equation*}
    \inf_{\tilde f_n}
    \sup_{P \,:\, \eta_P \in \Sigma(L,[0,1]^d)}
    \mathbb{E}[ R(\tilde f_n)] - 
    R^*_P
    \ge 
    C n^{-\frac{1}{d+2}}.
  \end{equation*}
\end{theorem}
The infimum above is taken over all possible learning algorithms $\tilde f_n$, i.e., mappings from $(\mathcal{X} \times \mathbb{B})^n$ to Borel measurable functions $\mathcal{X} \to \mathbb{B}$.
When $\hat f_n$ is an empirical risk minimizer (ERM) over $\HAC$
where $d = r$ for $k = n^{\frac{1}{d+2}}$
, this minimax rate is achieved.

\begin{theorem}
  \label{theorem: minimax result}
  Let $d \ge 1$ be fixed.
  Let $\hat f_n$ be an ERM over $\mathtt{HAC}(d,d,k)$ where $k = k(n) \asymp  n^{\frac{1}{d+1}}$. Then
  there exists a constant $C'$ such that
  \[
    \sup_{P \,:\, \eta_P \in \Sigma(L,[0,1]^d)}
    \mathbb{E}[ R(\hat f_n)] - 
    R^*_P
    \le
    C' n^{-\frac{1}{d+2}}.
  \]
\end{theorem}
\emph{Proof sketch (see Appendix~\ref{section: minimax result proof} for full proof).} 
We first show that the histogram classifier over the standard partition of $[0,1]^d$ into smaller cubes is an element of $\mathcal{C} := \mathtt{HAC}(d,d,k)$, thus
reducing the problem to proving minimax-optimality of the histogram classifier.
Previous work \cite{gyorfi2006distribution} Theorem 4.3 established this for the histogram \emph{regressor}.
The analogous result for the histogram \emph{classifier}, to the best of our knowledge, has not appeared in the literature and thus is included for completeness.

The neural network implementation of $\mathtt{HAC}(d,d,k)$ where $k \asymp n^{1/(d+2)}$ in \cref{theorem: minimax result} can be overparametrized.
Using either the 1- or the 2-hidden layer neural network implementations of Boolean functions as in Remark~\ref{remark: Boolean implementations}, the resulting HANN is overparametrized and has number of weights either $\ge k2^k$ or $\ge c^22^k/k$ respectively.
Both lower bounds on the number of weights are exponentially larger than $n$ meanwhile $\mathtt{VC}(\mathtt{HAC}(d,d,k)) = o(n)$.

\section{Empirical results}

In this section, we discuss experimental results of using HANNs for classifying synthetic and real datasets.
Our implementation uses TensorFlow \citep{abadi2016tensorflow} with the Larq \citep{larq} library for training neural networks with threshold activations.
Note that \cref{theorem: minimax result} holds for ERM over HANNs, which is intractable in practice.

\textbf{Synthetic datasets.}
We apply a HANN (model specification shown in \cref{fig: moons}-top left)  to the \textsc{moons} synthetic dataset with two classes with the hinge loss.

The heuristic for training networks with threshold activation can significantly affect the performance \citep{kim2019binaryduo}.
We consider two of the most popular heuristics: 
the straight-through-estimator (SteSign)
and 
the SwishSign, introduced by
\cite{hubara2017quantized}
and
\cite{darabi2020regularized}, respectively.
SwishSign reliably leads to higher validation accuracy (\cref{fig: moons}-bottom left), consistent with the finding of \cite{darabi2020regularized}. 
Subsequently, we use SwishSign and plot a learned decision boundary
in \cref{fig: moons}-right.
\begin{figure}[htpb]
  \centering
  \includegraphics[width=0.5\linewidth]{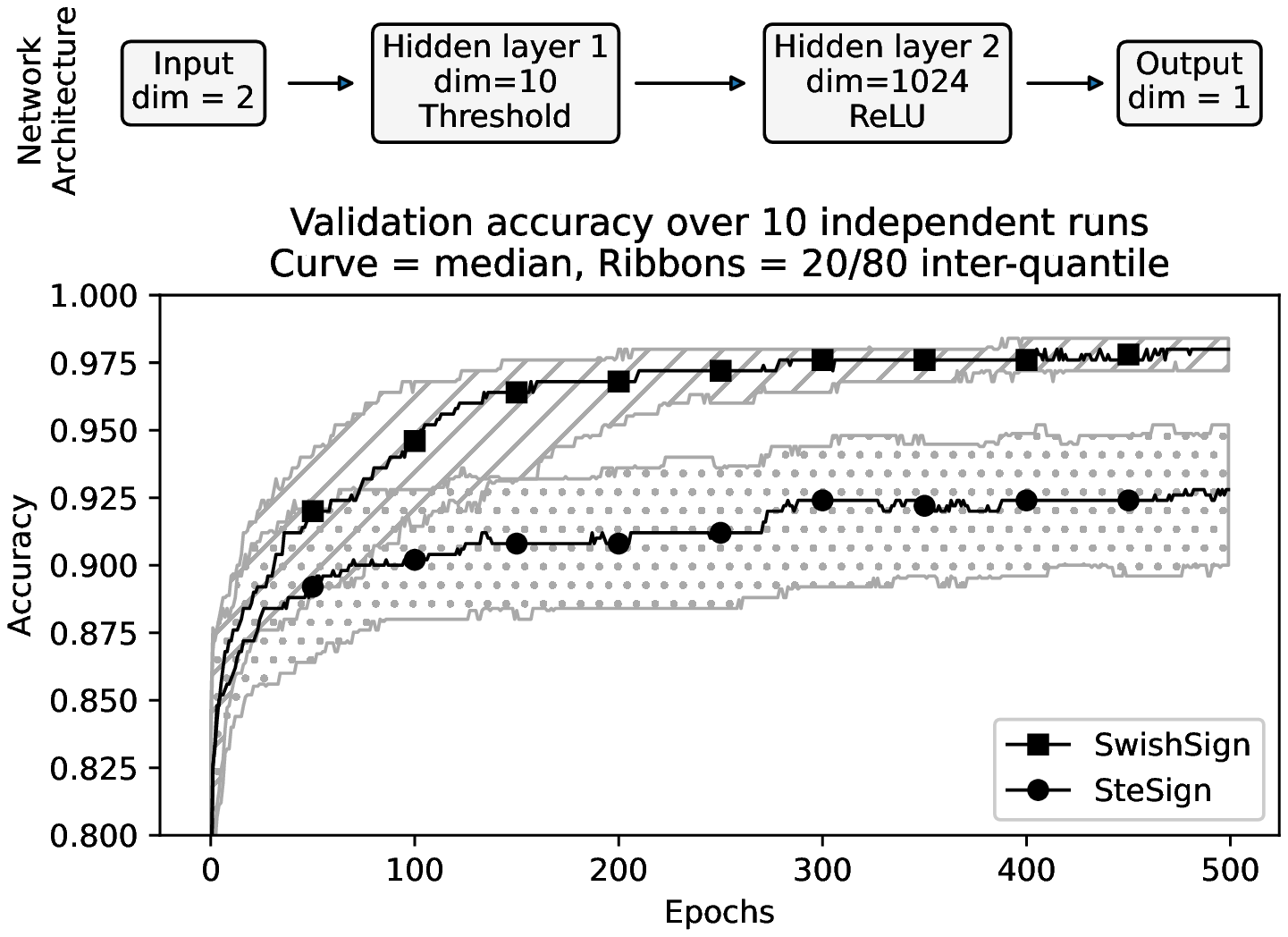}
  ~
  \includegraphics[width=0.48\linewidth]{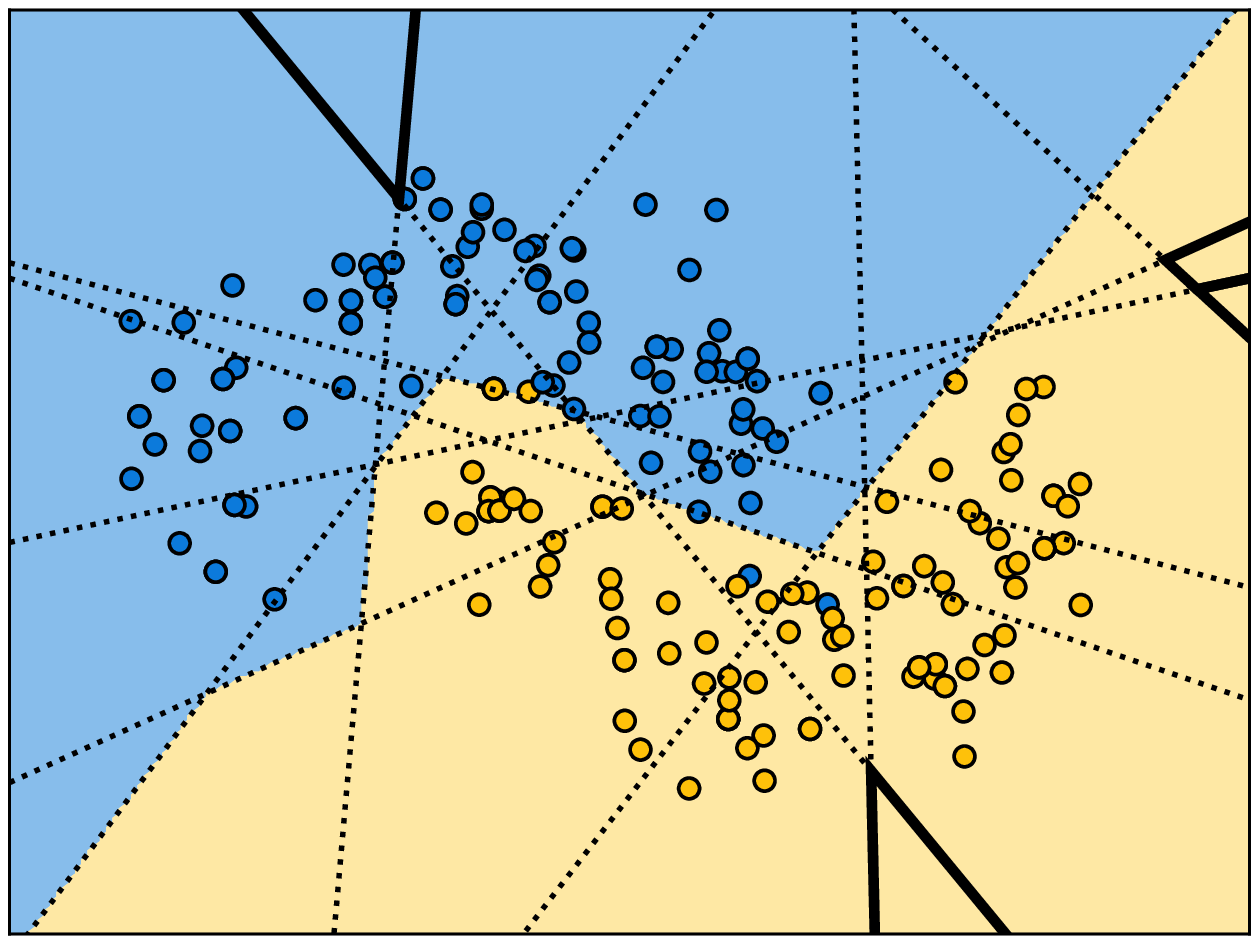}
  \caption{\emph{Top left.} Architecture of HANN used for the \textsc{moons} dataset.
    \emph{Bottom left.} Validation accuracies from 10 independent runs with random initialization and data generation.
    \emph{Right.}
    Data points (circles) drawn from \texttt{make\_moons} in \texttt{sklearn} colored by ground truth labels.
    The hyperplane arrangement is denoted by dotted lines.
  Coloring of the cells corresponds to the decision region of the trained classifier.
  A cell $\Delta$ is highlighted by bold boundaries if 1) no training data lies in $\Delta$ and 2) $\Delta$ does not touch the decision boundary.
}%
  \label{fig: moons}
\end{figure}

By \cite{mukherjee2017lower} Lemma 2.1, any Boolean function $\mathbb{B}^{k} \to \mathbb{B}$ can be implemented by a 1-hidden layer ReLU network with $2^k$ hidden nodes.
Here, the width of the hidden layer is $2^{10} = 1024$.
Thus, the architecture in \cref{fig: moons} can assign labels to the bold boundary cells arbitrarily without changing the training loss.
Nevertheless, the optimization appears to be biased toward a topologically simpler classifier.
This behavior is consistently reproducible. See \cref{fig: more moons}.

\textbf{Real-world datasets.}
\cite{klambauer2017self} introduced \emph{self-normalizing neural networks} (SNN) which were shown to outperform other neural networks on a panel of 121 UCI datasets.
Subsequently, \cite{wu2018improved} proposed the \emph{dendritic neural network} architecture, which further improved classification performance on this panel of datasets.
Following their works,
we evaluate the performance of HANNs on the 121 UCI datasets.

A crucial hyperparameter for HANN is $k$, the number of hyperplanes used.
We ran the experiments with $k \in \{15,100\}$ to test the hyperparameter's impact on accuracy.
The Boolean function $h$ is implemented as a 1-hidden layer residual network \citep{he2016deep} of width $1000$.

We use the same train, validation, and test sets from 
the public code repository of \cite{klambauer2017self}.
The reported accuracies on the held-out test set are based on the best performing model according to the  validation set.
The models will be referred to as \texttt{HANN15} and \texttt{HANN100}, respectively.
The results are shown in \cref{figure: empirical results}.
The accuracies of SNN and DENN are obtained from Table A1 in the supplemental materials of \cite{wu2018improved}.
Full details for the training and accuracy tables can be found in the appendix.

\begin{figure}[h]
\begin{center}
  \includegraphics[width = 1 \textwidth]{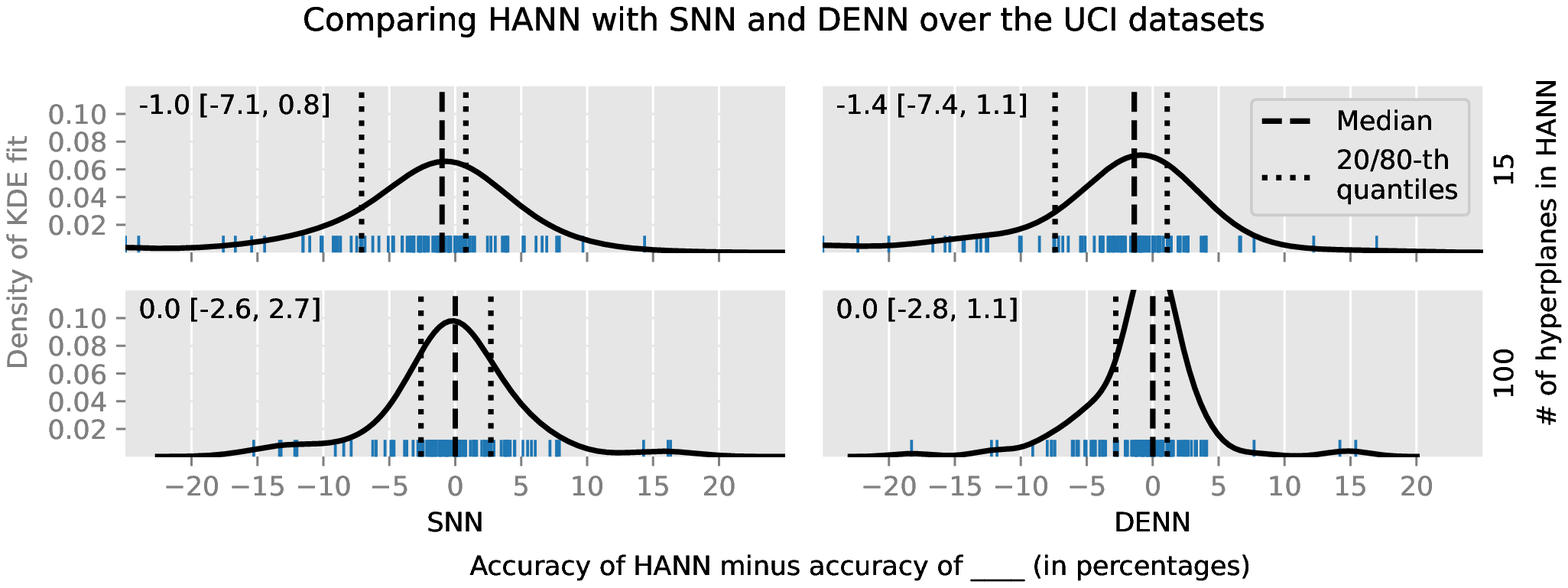}
\end{center}
\caption{Each blue tick above the x-axis represents a single dataset, where the x-coordinate of the tick is the difference of the accuracy of HANN and either SNN (left) or DENN (right) on the dataset.
  The solid black curves are kernel density estimates for the blue ticks.
The number of hyperplanes used by HANN is either 15 (top) or 100 (bottom).
The quantities shown in the top-left corner of each subplot are the median, 20-th and 80-th quantiles of the differences, respectively, rounded to 1 decimal place.}
\label{figure: empirical results}
\end{figure}

The \texttt{HANN15} model (top row of \cref{figure: empirical results}) already achieves median accuracy within 1.5\% of both SNN and DENN.
With the larger \texttt{HANN100} model (bottom row), the gap is reduced to zero.
The largest training set in this panel of datasets has size $77904$.
The \texttt{HANN15} and \texttt{HANN100} models use $\approx 10^4$ and $10^5$ weights, respectively.
By comparison, the average number of weights used by SNN is $\ge 5*10^5$, while the number of weights used by DENN is at least $\ge 2.5*10^5$.
Thus, all three models considered here, namely HANN, SNN and DENN, are overparametrized for this panel of datasets.

\section{Discussion}

We have introduced an architecture for which the VC theorem can be used to prove minimax-optimality of ERM over HANNs in an overparametrized setting with Lipschitz posterior.
To our knowledge, this is the first time VC theory has been used to analyze the performance of a neural network in the overparametrized regime.
Furthermore, the same architecture leads to state-of-the-art performance over a benchmark collection of unstructured datasets.



\subsubsection*{Reproducibility Statement}
All code for downloading and parsing the data, training the models, and generating plots in this manuscript are available at \url{https://github.com/YutongWangUMich/HANN}.
Complete proofs for all novel results are included in the main article or in an appendix.

\bibliography{references}
\bibliographystyle{iclr2022_conference}

\appendix

\section{Proofs}

\subsection{Proof of Proposition~\ref{proposition: canonical form}}
\label{section: canonical form}
  By definition, there exists $h \in \Bool{k}$, $\bW \in \mathbb{R}^{d \times k}$ of rank at most $r$, and $\biasb \in \mathbb{R}^k$ such that 
  $y_i = h(\sgn(\bW^\top x_i + \biasb))$.

  Now, let $j \in [k]$ be fixed.
  Since
  $|\bw_j^\top x_i + b_j| \ge 0$ for all $i \in [n]$, there exists a small perturbation $\tilde c_j$ of $b_j$ such that
  $|\bw_j^\top x_i + \tilde c_j| > 0$ for all $i \in [n]$.
  Now, let $\lambda_j := \min_{i\in [n]} |\bw_j^\top x_i + \tilde c_j|$ which is positive.
  Define $\bv_j := \bw_j/\lambda_j$ and $c_j = \tilde c_j/\lambda_j$, we have 
  $|\bv_j^\top x_i + c_j| \ge 1$ for all $i \in [n]$, as desired.
Note that $\mathrm{rank}(\bV) = \mathrm{rank}(\bW)$.
\hfill$\Box$

\subsection{Proof of Proposition~\ref{proposition: caratheodory theorem}}
\label{section: proof of caratheodory theorem}
  Let $g_i(x) = \ba_i^\top x - b_i$ for each $i \in [m]$ and $f(x) = \frac{1}{2} \|x\|^2_2$.
  Then $\nabla f(x) = x$ and $\nabla g_i(x) = \ba_i$.
  By definition, $x_I^*$ is a minimizer of 
  \[\min_{x \in \mathbb{R}^n} f(x) \mbox{ s.t. } g_i(x) \le 0,\, \forall i \in I,\]
  which is a convex optimization with strongly convex objective.
  Thus, the minimizer $x_I^*$ is unique and furthermore is the unique element $x$ of $\mathbb{R}^n$ satisfying the KKT conditions:
  \begin{equation*}
    x \in \mathcal{P}_I
    \mbox{ and }
    \exists \mbox { a set of nonnegative weights $\{\lambda_i\}_{i \in I}$ such that }
    -x = \sum_{i\in I} \lambda_{i} \ba_i.
  \end{equation*}
  
  Thus, $x_I^*$ can be equivalently characterized as
  the unique element of $x \in \mathbb{R}^n$ satisfying
  \begin{equation}
    x  \in \mathcal{P}_I
    \mbox{ and }
    -x \in \mathtt{coni}(\{\ba_i\}_{i \in I}).
  \end{equation}
  In particular, $x_{[m]}^* \in \mathcal{P}_{[m]}$ and 
  $-x_{[m]}^* \in \mathtt{coni}(\{\ba_i\}_{i \in [m]})$. By the Carath\'edory's theorem for the conical hull \citep{lovasz2009matching}, there exists $\underline{I} \subseteq [m]$ such that $|\underline{I}| = n$ and  $-x_{[m]}^* \in \mathtt{coni}(\{\ba_i\}_{i \in \underline{I}})$.
  Thus, for any $J \subseteq [m]$ such that $\underline{I} \subseteq J$, we have $-x_{[m]}^* \in \mathtt{coni}(\{\ba_i\}_{i\in J})$.
  Furthermore, 
  $J \subseteq [m]$ implies $\mathcal{P}_J \supseteq \mathcal{P}_{[m]}$.
  In particular, $x_{[m]}^* \in \mathcal{P}_J$.
  Putting it all together, we have $x_{[m]}^* \in \mathcal{P}_J$ and $-x_{[m]}^* \in \mathtt{coni}(\{\ba_i\}_{i \in J})$.
  By the uniqueness, we have $x_{J}^* = x_{[m]}^*$.
\hfill $\Box$

\subsection{Proof of \cref{theorem: minimax result}}
\label{section: minimax result proof}

In this proof, the constant $C$ does not depending on  $n$, and may change from line to line.

  We  fix a joint distribution $P$ such that $\eta_P \in \Sigma(L,[0,1]^d)$ throughout the proof. Thus, the notation for risks will omit the $P$ in their subscript, e.g., we write $\hat R_{n}(f)$ instead of $\hat R_{P,n}(f)$ and $R^*$ instead of $R^*_P$.
  Below, let $\beta>\alpha >0$ be constants such that $\alpha dn^{1/(d+2)} \le k \le \beta dn ^{1/(d+2)}$.
  Let $\tilde{k} := \lceil k/d \rceil$.

  Let $\mathcal{R}_1, \mathcal{R}_2, \dots , \mathcal{R}_{\tilde{k}^d}$ denote the hypercubes of side length $\ell = 1/\tilde{k}$ forming a partition of $[0,1]^d$.
  For each $i \in [\tilde{k}^d]$, let $\mathcal{R}_{i}^{-} := \{ x \in \mathcal{R}_i : \eta_P(x) < 1/2\}$ and $\mathcal{R}_{i}^{+} := \{ x \in \mathcal{R}_i: \eta_P(x) \ge 1/2\}$.

  Let $\tilde f : [0,1]^d \to \mathbb{B}$ be the classifier such that 
\begin{equation*}
  \tilde f(x) = 
  \begin{cases}
    +1 &: x \in \mathcal{R}_i, \,
    \int_{\mathcal{R}_i} \eta_P(x) dP(x)\ge \int_{\mathcal{R}_i} (1-\eta_P(x)) dP(x) 
    \\
    -1 &: x \in \mathcal{R}_i,\,
    \int_{\mathcal{R}_i} \eta_P(x) dP(x) <  \int_{\mathcal{R}_i} (1-\eta_P(x)) dP(x).
  \end{cases}
\end{equation*}
In other words, $\tilde f$ classifies all $x \in \mathcal{R}_i$ as $+1$ if and only if $P(Y=1 | X \in \mathcal{R}_i) \ge 1/2$.
This is commonly referred to as the \emph{histogram classifier} \citep{gyorfi2006distribution}.
It is easy to see that
\[
  P(\tilde f(X) \ne Y , X \in \mathcal{R}_i)
  =
  \min\left\{
    \int_{\mathcal{R}_i} (1-\eta_P(x)) dP(x),
  \int_{\mathcal{R}_i}\eta_P(x) 
dP(x) \right\}
\]
For the remainder of this proof, we write ``$\sum_i$'' to mean ``$\sum_{i \in [\tilde{k}^d]}$''.
Thus,
\[
  R(\tilde f)
  =
  \sum_{i}
  P(\tilde f(X) \ne Y , X \in \mathcal{R}_i)
  =
  \sum_i \min \left\{ 
  \int_{\mathcal{R}_i} (1-\eta_P(x)) dP(x) ,
\int_{\mathcal{R}_i} \eta_P(x) dP(x)\right\}.
\]

Next, we note that $\tilde f \in \mathtt{HAC}(d,d,k)$. To see this, let $j \in [d]$. Take $H_{j1},\dots, H_{j(\tilde{k}-1)} \subseteq \mathbb{R}^d$ to be the hyperplanes perpendicular to the $j$-th coordinate where, for each $\ell \in [\tilde{k}]$, $H_{j\ell}$ intersects the $j$-th coordinate axis at $\ell/\tilde{k}$.
Consider the hyperplane arrangement consisting of all $\{H_{j\ell}\}_{j \in [d], \ell \in [\tilde{k}-1]}$ and let $\{C_1,C_2,\dots \}$ be its cells.
Then $\{C_1 \cap [0,1]^d, C_2 \cap [0,1]^d,\dots\} = \{\mathcal{R}_1,\dots, \mathcal{R}_{\tilde{k}^d}\}$ is the partition of $[0,1]^d$ by $1/\tilde{k}$ side length hypercubes.
See \cref{fig: regular grid}.

\begin{figure}[htpb]
  \centering
\begin{tikzpicture}[scale=0.9,>=stealth]
        \fill[color=gray!20] (-1.48,-1.48) rectangle (1.5,1.5);
        \fill[pattern=dots,  pattern color=gray] (-0.5,0.5) rectangle (0.5,2);
        \draw[->, thick](-1.5,-2) -- (-1.5,2);
        \draw[->, thick](-2,-1.5) -- (2,-1.5);
\foreach \m  in {0,1}
 \draw[ultra thick, lightgray, xshift=-0.75] (-0.5+\m,-2) -- (-0.5+\m,2) ;
\foreach \m in {0,1}
 \draw[ultra thick,xshift=0.75] (-0.5+\m,-2) -- (-0.5+\m,2) ;
\foreach \m in {0,1}
 \draw[ultra thick, lightgray,yshift=-0.75] (-2,-0.5+\m) -- (2,-0.5+\m) ;
\foreach \m in {0,1}
 \draw[ultra thick,yshift=0.75] (-2,-0.5+\m) -- (2,-0.5+\m) ;
\node at (-0.5+0,-2.3) {$H_{11}$};
\node at (-0.5+1,-2.3) {$H_{12}$};
\node at (-2.4,-0.5+0) {$H_{21}$};
\node at (-2.4,-0.5+1) {$H_{22}$};
\draw [thick,decorate,decoration={brace,amplitude=5pt,mirror,raise=4pt},yshift=0pt]
(2.1,-0.5) -- (2.1,0.5) node [black,midway,xshift=0.8cm] {\footnotesize
$1/3$};
\foreach \m in {0,1}
  \foreach \n in {0,1}
     \node[circle,fill=black,scale=0.6] at (-0.5+\m,-0.5+\n) {};
\end{tikzpicture}
\caption{Partition of $[0,1]^d$ into $1/\tilde{k}$ hypercubes via arrangement of $d(\tilde{k}-1)$ hyperplanes, where $d=2$ and $\tilde{k}=3$.
Shaded region is $[0,1]^d$. Dotted region is a cell of the hyperplane arrangement. }%
  \label{fig: regular grid}
\end{figure}

Let $\bW$ be the matrix of normal vectors and $\biasb$ be the vector of offsets representing this hyperplane arrangement, which requires 
$d(\tilde{k}-1) = d (\lceil k/d \rceil - 1) \le d( k/d) = k$ hyperplanes.
Since $\tilde f$ is constant on $\mathcal{R}_i$,
there exists a Boolean function $h \in \Bool{k}$ such that $h \circ q_{\bW,\biasb}|_{[0,1]^d} = \tilde f$.
From this, we conclude that $\tilde f \in 
\mathtt{HAC}(d,d,k)$.

Thus
$\hat R_n(\hat f_n) - \hat R_n(\tilde f) \le 0$ and so
\begin{align*}
  R(\hat f_n) - R^*
  &=
  R(\hat f_n) - \hat R_n(\hat f_n)
  + \underbrace{\hat R_n(\hat f_n) - \hat R_n(\tilde f)}_{\le 0}
  +\hat R_n(\tilde f) - R(\tilde f)
  + R(\tilde f) - R^*
  \\
  & \le
  \underbrace{R(\hat f_n) - \hat R_n(\hat f_n)}_{\mbox{\tiny{Term 1}}}
  +
  \underbrace{\hat R_n(\tilde f) - R(\tilde f)}_{\mbox{\tiny{Term 2}}}
  + \underbrace{R(\tilde f) - R^*}_{\mbox{\tiny{Term 3}}}.
\end{align*}
We now bound Terms 1 and 2 using the uniform deviation bound.
From \cref{theorem: VC dimension bound for HAC}, we know that there exists a constant $C$ independent of $n$  such that
  \[
    \mathtt{VC}(\mathtt{HAC}(d,d,k))
    \le 8\cdot\left(
      k(d+1) + k(d+1)(1+ \lceil \log_2 (k) \rceil)  + \binom{k}{\le d}
    \right)
    \le C k^d.
  \]
  Thus, by \cref{theorem: uniform deviation bound} with $\delta = 1/(2n)$ and a union bound, 
  with probability at least $1-1/n$ 
\begin{equation}
  \label{equation: uniform deviation bound applied}
 \max
 \left\{ 
   |\hat {R}_{n}(\hat f_n) - R(\hat f_n)| ,\,
   |\hat {R}_{n}(\tilde f) - R(\tilde f)| 
 \right\}
  \le 
  C
  \sqrt{
    \frac{ k^d + \log(n)}{n}
  }
\end{equation}
for some $C >0$.

Next, we focus on Term 3.
Recall that
\begin{equation*}
  R^* =
  \int_{[0,1]^d}
  \min \{ \eta_P(x),  1- \eta_P(x) \}  dP(x)
      =
      \sum_{i}
      \int_{\mathcal{R}_i}
  \min \{ \eta_P(x),  1- \eta_P(x) \}  dP(x)
\end{equation*}
and that
\begin{equation*}
  R(\tilde f) = 
  \sum_{i}
  \min\left\{
    \int_{\mathcal{R}_i} \eta_P(x) dP(x),
    \int_{\mathcal{R}_i} 1 - \eta_P(x) dP(x)
\right\}.
\end{equation*}
Fix some $i \in [k^d]$.
Our goal now is to bound the difference between the $i$-th summands in the above expressions for $R(\tilde f)$ and $R^*$:
\begin{equation}
  \label{equation: minimax proof - histogram classifier - difference}
  \min\left\{
    \int_{\mathcal{R}_i} \eta_P(x) dP(x),
    \int_{\mathcal{R}_i} 1 - \eta_P(x) dP(x)
\right\}
-
      \int_{\mathcal{R}_i}
  \min \{ \eta_P(x),  1- \eta_P(x) \}  dP(x).
\end{equation}
First, consider the case that 
\begin{equation}
  \label{equation: minimax proof - histogram classifier - assumption}
  \min\left\{
    \int_{\mathcal{R}_i} \eta_P(x) dP(x),
    \int_{\mathcal{R}_i} 1 - \eta_P(x) dP(x)
\right\}
=
    \int_{\mathcal{R}_i} \eta_P(x) dP(x).
\end{equation}
We claim that there must exist $x_0 \in \mathcal{R}_i$ such that $\eta_P(x_0) \le 1/2$.
Suppose $\eta_P(x) > 1/2$ for all $x \in \mathcal{R}_i$.
Then $\eta_P(x) > 1/2 > 1 - \eta_P(x)$. Since $\eta_P(x)$ is continuous, this would contradict \cref{equation: minimax proof - histogram classifier - assumption}.

Continue assuming \cref{equation: minimax proof - histogram classifier - assumption}, we further divide into two subcases: (1) $\eta_P(x) \le 1/2$ for all $x \in \mathcal{R}_i$, and (2)
there exists some $x_1 \in \mathcal{R}_i$ such that $\eta_P(x_1) > 1/2$.

Under subcase (1), $\min \{ \eta_P(x),  1- \eta_P(x) \} = \eta_P(x)$ for all $x \in \mathcal{R}_i$ in which case
$
  \cref{equation: minimax proof - histogram classifier - difference}
    =0.
$

Under subcase (2), since $\eta_P(x_0) \le 1/2 < \eta_P(x_1)$, we know by the intermediate value theorem that there must exist $x' \in \mathcal{R}_i$ such that $\eta_P(x') = 1/2$.
Now,
\begin{align*}
  \cref{equation: minimax proof - histogram classifier - difference}
  &=
  \int_{\mathcal{R}_i} (\eta_P(x)
  -
  \min \{ \eta_P(x),  1- \eta_P(x) \})  dP(x)
  \\
  &\le
  \int_{\mathcal{R}_{i}^+} (\eta_P(x)
  -
  \min \{ \eta_P(x),  1- \eta_P(x) \})  dP(x)
  \\
  &
  \qquad 
  +
  \int_{\mathcal{R}_{i}^-} (\eta_P(x)
  -
  \min \{ \eta_P(x),  1- \eta_P(x) \})  dP(x)
  \\
  &
  =
  \int_{\mathcal{R}_{i}^+} (\eta_P(x)
  -
  (1- \eta_P(x) )  dP(x) \qquad \because \mbox{Definition of $\mathcal{R}_i^{\pm}$}
  \\
  &
  \qquad 
  +
  \int_{\mathcal{R}_{i}^-} (\eta_P(x)
  -
  \eta_P(x))  dP(x)
   \\
  &=
  \int_{\mathcal{R}_{i}^+} (2\eta_P(x)-1) dP(x)
  \\
  &=
  2\int_{\mathcal{R}_{i}^+} (\eta_P(x)-\eta_P(x')) dP(x) \qquad \because 2\eta_P(x')=1
  \\
  &\le
  2L\int_{\mathcal{R}_{i}^+} \|x-x'\|_2 dP(x) 
  \\
  &\le
  2L \sqrt{d}  \Pr(\mathcal{R}_i)/\tilde{k} \qquad \because \|x - x'\|_2 \le \sqrt{d} \|x- x'\|_1 \le \sqrt{d}({1}/\tilde{k})
  \\
  &\le
  2L {d}^{3/2}  \Pr(\mathcal{R}_i)/{k} \qquad \because 1/\tilde{k} = 1/\lceil k/d \rceil \le 1/ (k /d) = d/k.
\end{align*}

Thus, under assumption \cref{equation: minimax proof - histogram classifier - assumption}, we have proven that $\cref{equation: minimax proof - histogram classifier - difference} \le 2L {d}^{3/2} /{k}$.
For the other assumption, i.e., the minimum in \cref{equation: minimax proof - histogram classifier - assumption} is attained by
$\int_{\mathcal{R}_i} 1 - \eta_P(x) dP(x)$, a completely analogous argument again shows that $\cref{equation: minimax proof - histogram classifier - difference} \le 2L {d}^{3/2} /{k}$.

Putting it all together, we have
\begin{equation}
  \label{equation: histogram classifier - bound}
  R(\tilde f_n) - R^*
  \le
    2
  L
  {d}^{3/2}
  \sum_{i}
  P(\mathcal{R}_i)
/{k}
  =
    2
  L
  {d}^{3/2}
/{k}.
\end{equation}

We have shown that,
with probability at least $1-1/n$,
\begin{align*}
  R(\hat f_n) - R^*
  \le
  C \sqrt{ \frac{k^d + \log (n)}{n}}
  + \frac{2L {d}^{3/2}}{k}.
\end{align*}
Using 
$
\alpha d n^{1/(d+2)} \le 
k \le \beta dn^{1/{(d+2)}}$, we have with probably at least $1-1/n$ that
\begin{align*}
  R(\hat f_n) - R^*
  &\le
  C \sqrt{ \frac{k^d + \log (n)}{n}}
  + \frac{2L {d}^{3/2}}{k}
  \\
  &\le
  C \sqrt{ \frac{(\beta d)^d n^{d/(d+2)} + \log (n)}{n}}
  + \frac{2L {d}^{3/2}}{\alpha d n^{1/(d+2)}}
  \\
  &\le
  C \left(\sqrt{ \frac{n^{d/(d+2)} }{n}}
  + 
  n^{-1/(d+2)}
  \right)
  \qquad \because
  \log(n) = o(n^{1/{d+2}})
  \\
  &=
  C \left(\sqrt{ {n^{-2/(d+2)} }}
  + 
  n^{-1/(d+2)}
  \right)
\\
  & \le 
  C n^{-\frac{1}{d+2}}.
\end{align*}
Taking expectation, we have 
$\mathbb{E}[ R(\hat f_n)] - R^*
\le 
(1 - 1/n)
  C n^{-\frac{1}{d+2}}
  +
  1/n\cdot 1
  \le 
  C n^{-\frac{1}{d+2}}$.
\hfill $\Box$

\section{Training details}
\textbf{Data preprocessing.} The pooled training and validation data is centered and standardized using the \texttt{StandardScaler} function from sklearn.
The transformation is also applied to the test data, using the centers and scaling from the pooled training and validation data:

\begin{lstlisting}[language=python]
scaler = StandardScaler().fit(X_train_valid)
X_train_valid = scaler.transform(X_train_valid)
X_test = scaler.transform(X_test)
\end{lstlisting}

If the feature dimension and training sample size are both $>50$, then the data is dimension reduced to 50 principal component features:
\begin{lstlisting}[language=python]
if min(X_train_valid.shape) > 50:
    pca = PCA(n_components = 50).fit(X_train_valid)
    X_train_valid = pca.transform(X_train_valid)
    X_test = pca.transform(X_test)
\end{lstlisting}
Note that this is equivalent to freezing the weights between the Input and the Latent layer in \cref{fig: HAC}.

\textbf{Validation and test accuracy.}
Every 10 epochs, the validation accuracy during the past 10 epochs are averaged. A smoothed validation accuracy is calculated as follows:

\begin{lstlisting}[language=python]
val_acc_sm = (1-sm_param)*val_acc_sm + sm_param*val_acc_av
## Variable description:
# sm_param = 0.1
# val_acc_av = average of the validation in the past 10 epochs
# val_acc_sm = smoothed validation accuracy
\end{lstlisting}
The predicted test labels is based on the snapshot of the model at the highest smoothed validation accuracy, at the end once max epochs is reached.

\textbf{Heuristic for coarse gradient of the threshold function.} We use the SwishSign from the Larq library \citep{larq}.
\begin{lstlisting}[language=python]
# import larq as lq
qtz = lq.quantizers.SwishSign()
\end{lstlisting}

\textbf{Dropout.} During training, dropout is applied to the Boolean output of the threshold function, i.e, the variables $\mathtt{B}_1,\mathtt{B}_2,\dots,\mathtt{B}_k$ in \cref{fig: HAC}.
This improves generalization by preventing the training accuracy from reaching $100\%$.

\begin{lstlisting}[language=python]
# from tensorflow.keras.layers import Dense, Dropout
hyperplane_enc = Dense(n_hyperplanes, activation = qtz)(inputs)
hyperplane_enc = Dropout(dropout_rate)(hyperplane_enc)
\end{lstlisting}

\textbf{Implementation of the Boolean function.}
For the Boolean function $h$, we use a 1-hidden layer residual network \citep{he2016deep} 
with $1000$ hidden nodes:
\begin{lstlisting}[language=python]
# from tensorflow.keras.layers import Dense, Add
# output_dim = num_classes
n_hidden = 1000
hidden = Dense(n_hidden, activation="relu")(hyperplane_enc)
out_hidden = Dense(output_dim, activation = "linear")(hidden)
out_skip = Dense(output_dim, activation = "linear")(hyperplane_enc)
outputs = Add()([out_skip,out_hidden])
\end{lstlisting}

\textbf{Hyperparameters.}
\texttt{HANN15} is trained with a hyperparameter grid of size 3 where
only the dropout rate is tuned.
The hyperparameters are summarized in Table~\ref{table: HANN15 grid}.
The model with the highest smoothed validation accuracy is chosen.

The model \texttt{HANN15} is trained with the following hyperparameters:
\begin{table}[H]
  \caption{\texttt{HANN15} model and training hyperparameter grid}
\label{table: HANN100 grid}
\begin{center}
\begin{small}
\begin{sc}
\begin{tabular}{lc}
\toprule
Optimizer      & SGD     \\
Learning rate & $0.01$\\
Dropout rate & $\{0.1,0.25,0.5\}$\\
Minibatch size & 128\\
Boolean function & 
\blap{
1-hidden layer resnet\\
with 1000 hidden nodes}
\\[15pt]
Epochs & \blap{500 \textsc{miniboone}\\ 5000 for all others}\\[10pt]
\bottomrule
\end{tabular}
\end{sc}
\end{small}
\end{center}
\vskip -0.1in
\end{table}

For \texttt{HANN100}, we only used 1 set of hyperparameters.

\begin{table}[H]
  \caption{\texttt{HANN100} model and training hyperparameter}
\label{table: HANN15 grid}
\begin{center}
\begin{small}
\begin{sc}
\begin{tabular}{lc}
\toprule
Optimizer      & SGD     \\
Learning rate & $0.01$\\
Dropout rate & $0.5$\\
Minibatch size & 128\\
Boolean function & 
\blap{
1-hidden layer resnet\\
with 1000 hidden nodes}
\\[15pt]
Epochs & \blap{500 \textsc{miniboone}\\ 5000 for all others}\\[10pt]
\bottomrule
\end{tabular}
\end{sc}
\end{small}
\end{center}
\vskip -0.1in
\end{table}

\newpage

\section{Additional plots}

\textbf{Multiclass hinge versus cross-entropy loss.}
\cref{figure: empirical results - OCE} shows the accuracy differences when the Weston-Watkins hinge loss is used.
Compared to the results shown in \cref{figure: empirical results}, the performance for \texttt{HANN100} is slightly worse and the performance for \texttt{HANN15} is slightly better.

\begin{figure}[H]
\begin{center}
  \includegraphics[width = 1 \textwidth]{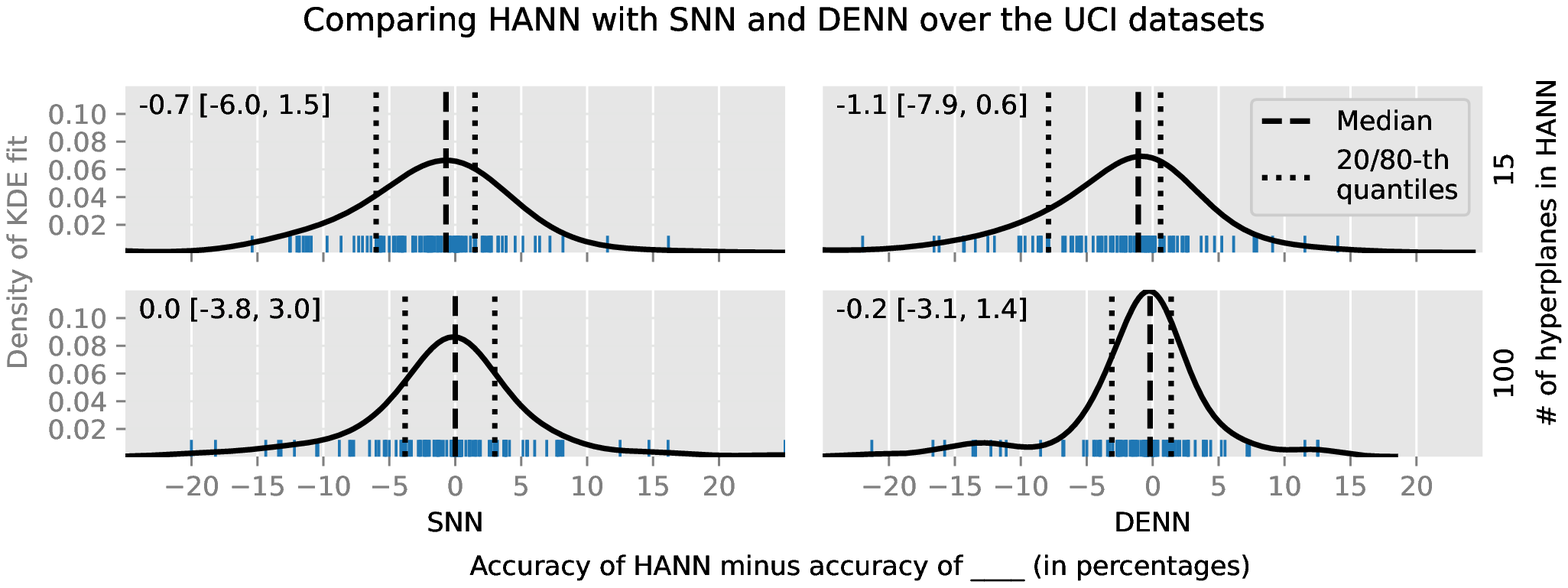}
\end{center}
\caption{Each blue tick above the x-axis represents a single dataset, where the x-coordinate of the tick is the difference of the accuracy of HANN and either SNN (left) or DENN (right) on the dataset.
The number of hyperplanes used by HANN is either 15 (top) or 100 (bottom).
The quantities shown in the top-left corner of each subplot are the median, 20-th and 80-th quantiles of the differences, respectively, rounded to 1 decimal place.}
\label{figure: empirical results - OCE}
\end{figure}

\newpage

\textbf{Implicit bias for low complexity decision boundary.}
In \cref{fig: more moons}, we show additional results ran with the same setting for the \textsc{moons} synthetic dataset as in the Empirical Results section.
From the perspective of the training loss, the label assignment in the bold-boundary regions is irrelevant.
Nevertheless, the optimization consistently appears to be biased toward the geometrically simpler classifier, despite the capacity for fitting complex classifiers.

\begin{figure}[htpb]
  \centering
  \includegraphics[width=0.48\linewidth]{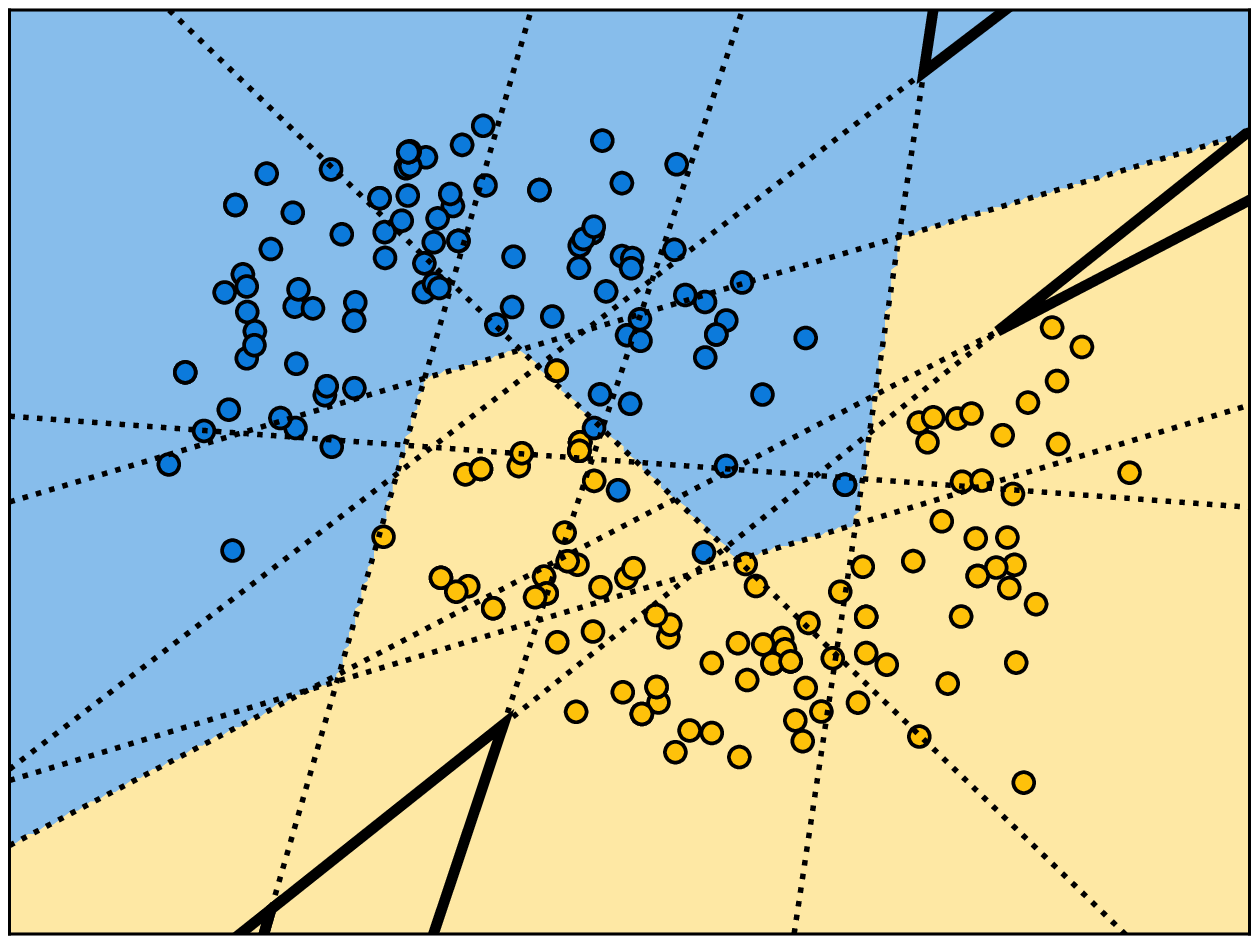}
  ~
  \includegraphics[width=0.48\linewidth]{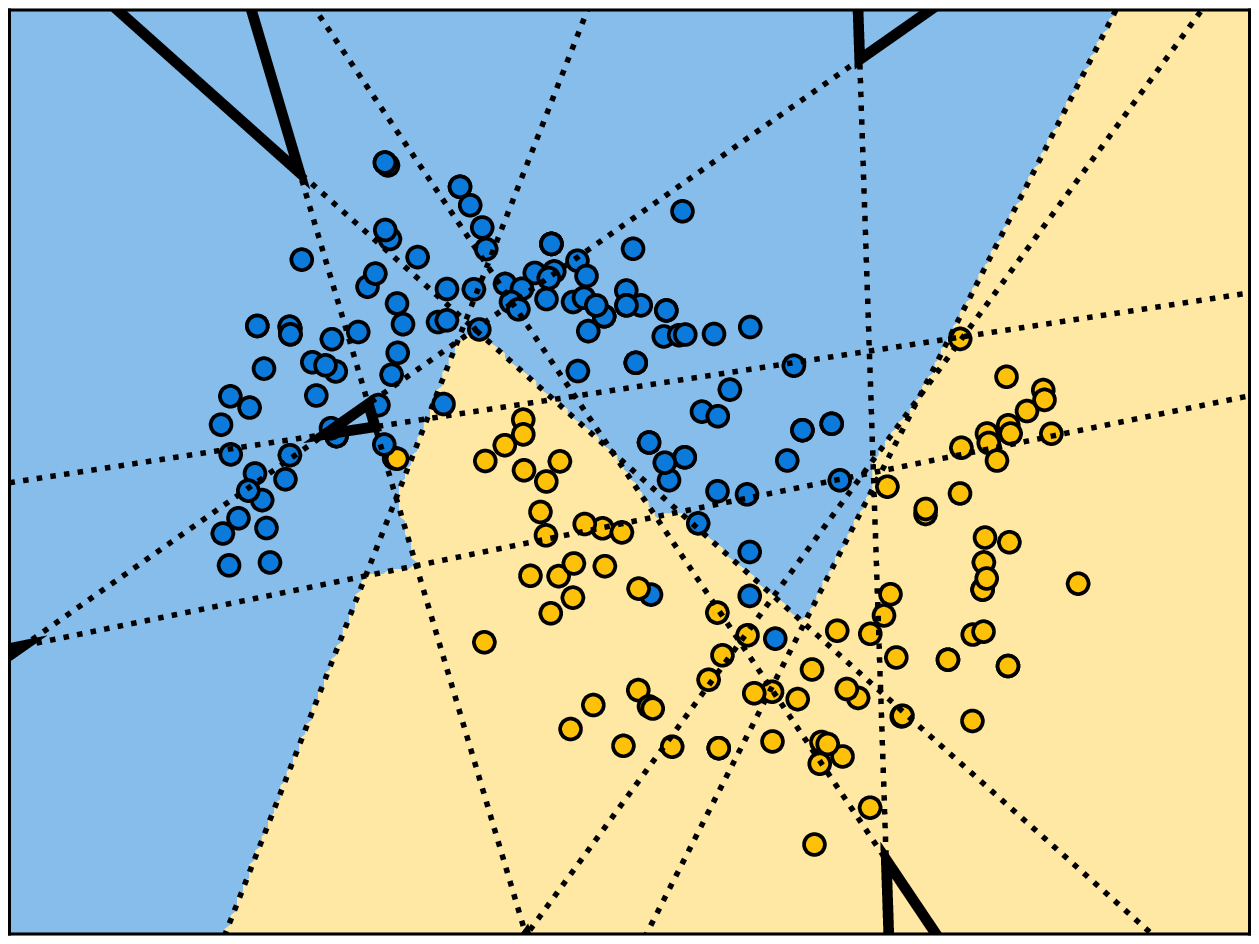}
  \includegraphics[width=0.48\linewidth]{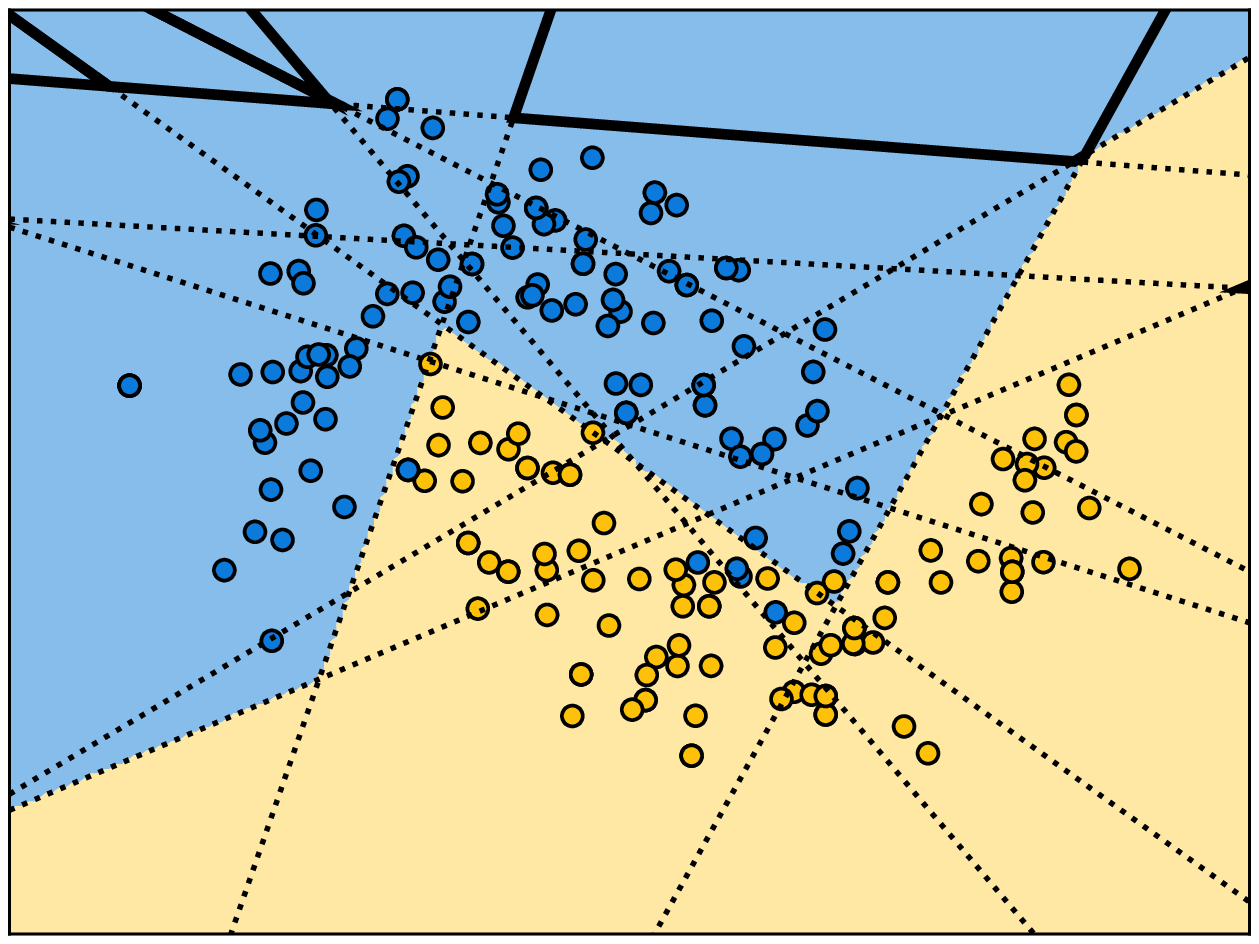}
  ~
  \includegraphics[width=0.48\linewidth]{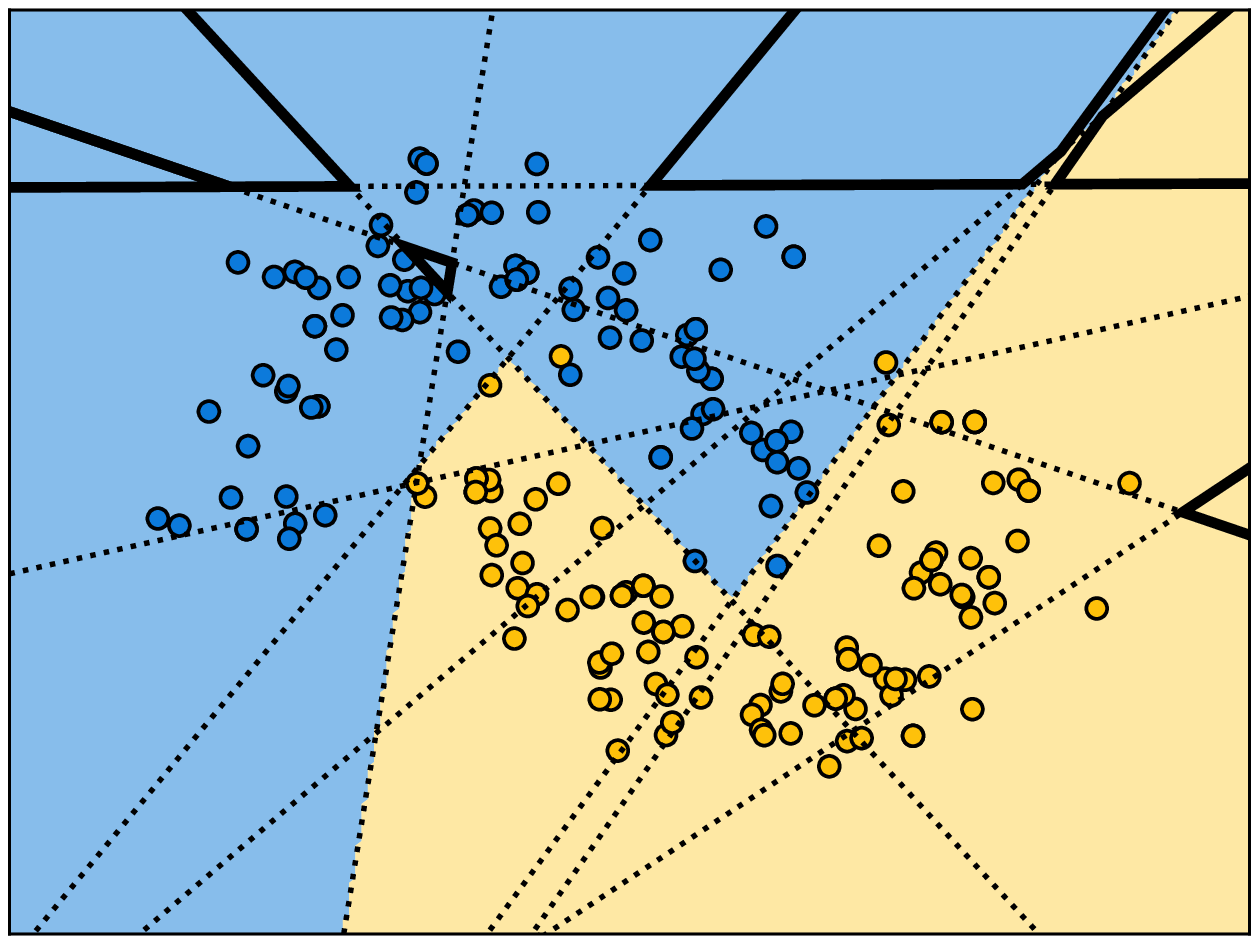}
  \caption{Four independent runs of HANN on the \textsc{moons} synthetic dataset.    
    Data points (circles) drawn from \texttt{make\_moons} in \texttt{sklearn} colored by ground truth labels.
    The hyperplane arrangement is denoted by dotted lines.
  Coloring of the cells corresponds to the decision region of the trained classifier.
  A cell $\mathcal{C}$ is highlighted by bold boundaries if 1) no training data lies in $\mathcal{C}$ and 2) $\mathcal{C}$ does not touch the decision boundary.
}%
  \label{fig: more moons}
\end{figure}

\section{Table of accuracies}
Below is the table of accuracies used to make \cref{figure: empirical results}.
Available as a csv file here: \url{https://github.com/YutongWangUMich/HANN/blob/main/accuracy_table.csv}
\input{figures/table_OCE.tex}
\end{document}

%% file: math_commands.tex
\usepackage{amsmath,amsfonts,bm}

\def\eqref#1{equation~\ref{#1}}

\def\1{\bm{1}}

\DeclareMathAlphabet{\mathsfit}{\encodingdefault}{\sfdefault}{m}{sl}
\SetMathAlphabet{\mathsfit}{bold}{\encodingdefault}{\sfdefault}{bx}{n}



%% file: more_math_commands.tex
\newcommand\sgn{\sigma_{\mathtt{sgn}}}

\newcommand\HAC{\mathtt{HAC}(d,r,k)}
\newcommand\bW{\mathbf{W}}
\newcommand\bV{\mathbf{V}}
\newcommand\bw{{w}}
\newcommand\bv{{v}}
\newcommand\ba{{a}}
\newcommand\biasb{{b}}
\newcommand\biasc{{c}}
\newcommand\bs{{s}}
\newcommand\Bool[1]{\mathtt{Bool}_{#1}}

%% file: diagram.tex
\begin{figure}[htpb]
  \centering

\tikzset{%
  every neuron/.style={
    circle,
    draw,
    scale=.7,
    inner sep = 0,
    minimum size=0.6cm
  },
  neuron missing/.style={
    draw=none, 
    scale=4,
    text height=0.333cm,
    execute at begin node=\color{black}$\vdots$
  },
}

\begin{tikzpicture}[x=1.5cm, y=1.5cm, >=stealth, scale=0.9, every node/.style={transform shape}]


\foreach \m/\l [count=\y] in {1,2,3,4}
\node [every neuron/.try, neuron \m/.try,scale=1.1] (input-\m) at (1,0.5*\y+.25) {$\mathtt{X}_\y$};

\foreach \m [count=\y] in {1,2}
  \node [every neuron/.try, neuron \m/.try,scale=1.1 ] (latent-\m) at (2,\y) {\includegraphics[width=0.33cm]{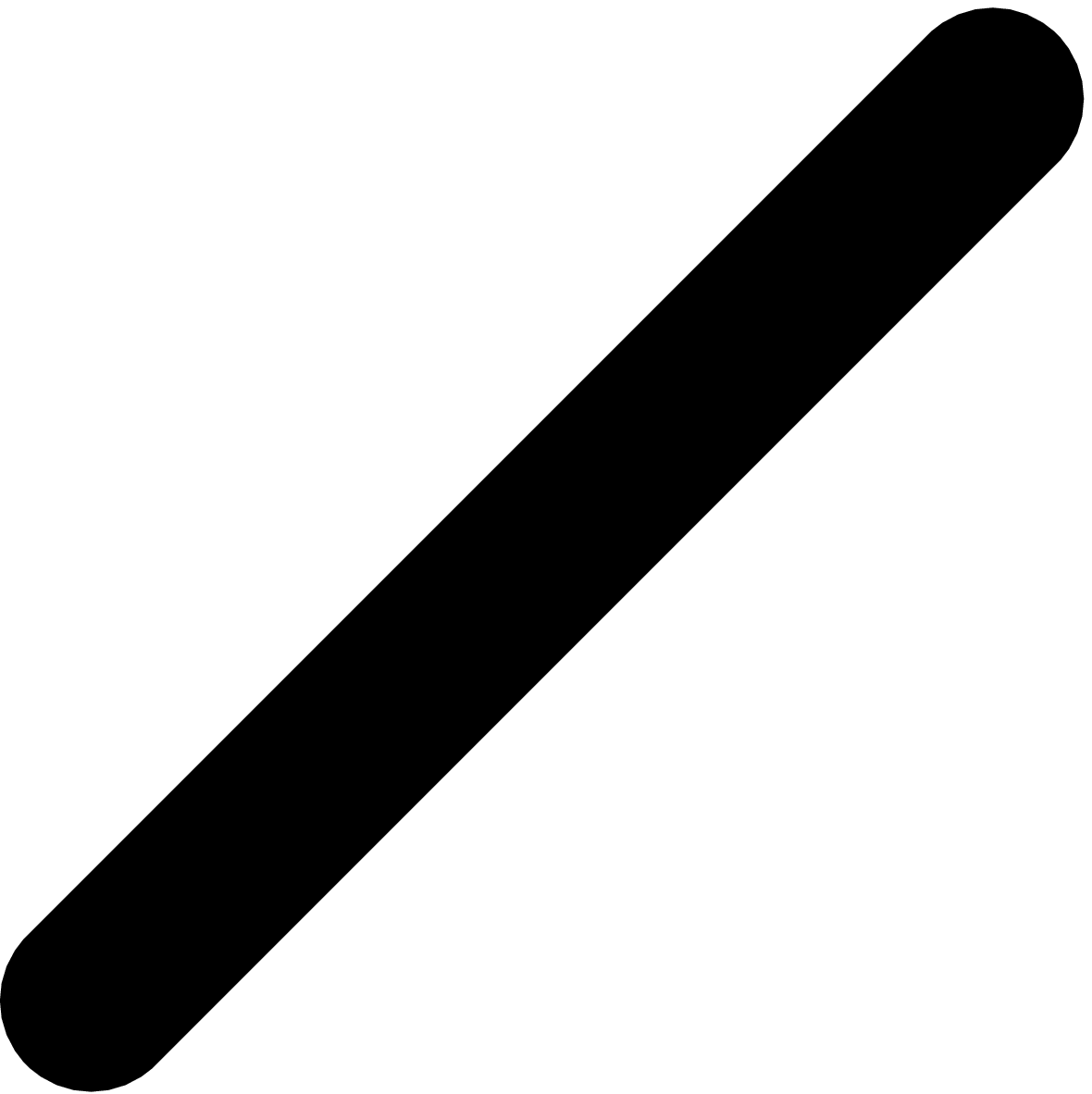}};

\foreach \m [count=\y] in {1,2,3}
  \node [every neuron/.try, neuron \m/.try,scale=1.1 ] (boolean-\m) at (3,0.75*\y) {\includegraphics[width=0.39cm]{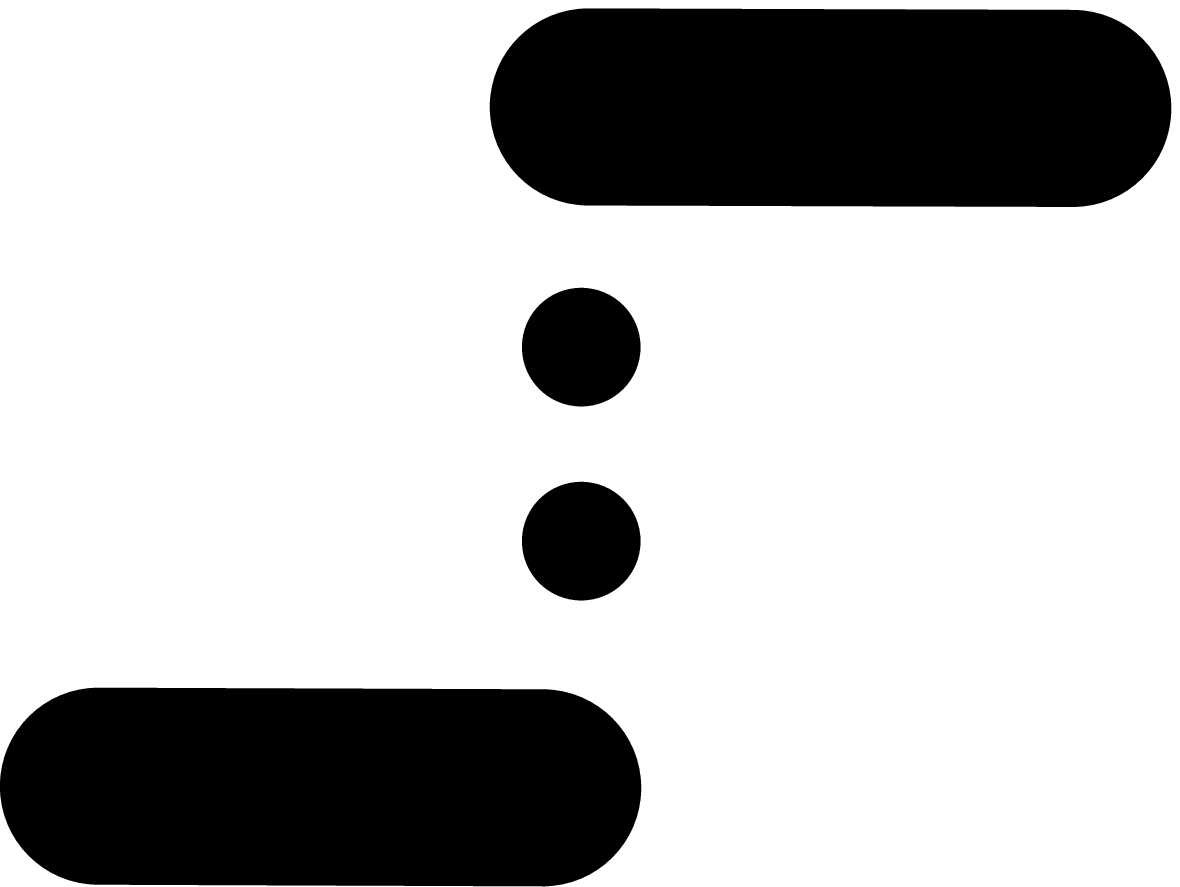}};



\foreach \l [count=\i] in {1,2,3} 
  \draw [->] (boolean-\i) -- ++(1,0)
    node [above, midway] {$\mathtt{B_\l}$};

\foreach \i in {1,...,4}
  \foreach \j in {1,...,2}
    \draw [->] (input-\i) -- (latent-\j);

\foreach \i in {1,...,2}
  \foreach \j in {1,...,3}
    \draw [->] (latent-\i) -- (boolean-\j);

  \node [align=center, above=0.9ex] at (input-4.north) {Input $\mathbb{R}^d$};
  \node [align=center, below] at (latent-1.south) {Latent $\mathbb{R}^r$};
  \node [align=center, above=1.2ex] at (boolean-3.north) {Boolean $\mathbb{B}^k$};
\draw[thick,draw=black] (4,0.5) rectangle ++(1.4,2);
  \node [align=center, above] at (4.7,2.5) {$h: \mathbb{B}^k \to \mathbb{B}$};
  \node [align=center] at (4.75,1.5) {
      $\begin{array}{l|l}
        \mathtt{B_1B_2B_3} & \mathtt{Y}\\
        \hline
        \mathtt{---} & \mathtt{-}\\
        \mathtt{--+} & \mathtt{+}\\
        \mathtt{-+-} & \mathtt{+}\\
        \vdots & \vdots\\
        \mathtt{+++} & \mathtt{-}
      \end{array}$
    };
  \draw [draw=gray] (5.4,2.5) -- (7,1.5);
  \draw [draw=gray] (5.4,0.5) -- (6,0.5);
  \draw [draw=gray] (5.4,1.8) -- (6,1.5);
  \fill[white] (5.45,2.3) rectangle ++(1,-0.65);
  \draw [->] (5.4,2) -- (6.5,2)
    node [below, midway] {Output}
    node [above, midway] {$\mathtt{Y}$};
\draw[thick,draw=lightgray] (6,1.5) rectangle ++(1,-1);
\foreach \m/\l [count=\y] in {1,2,3}
\node [every neuron/.try, neuron \m/.try,scale=0.4] (input0-\m) at (6.25,0.25*\y+.5) {};

\foreach \m/\l [count=\y] in {1,2,3,4}
\node [every neuron/.try, neuron \m/.try,scale=0.4] (hidden0-\m) at (6.5,0.2*\y+.5) {};

\node [every neuron/.try, neuron /.try,scale=0.4] (output0) at (6.75,1) {};

\node at (6.8,0.7) {$h_\theta$};
\foreach \l [count=\i] in {1,2,3} 
  \draw [-] (input0-\i) -- ++(-.18,0);

\foreach \i in {1,2,3}
  \foreach \j in {1,...,4}
    \draw [-] (input0-\i) -- (hidden0-\j);

\foreach \i in {1,...,4}
    \draw [-] (hidden0-\i) -- (output0);
\draw [-] (output0) -- ++(.18,0);
\draw[rounded corners,thick,draw=lightgray] (6.8+\x1,2.95+\y1) rectangle ++(1.15,-0.95); 
\node [right,scale = 0.75]at (6.9+\x1,2.8+\y1) {Activations};
  \node [right, scale=0.75] at (7.1+\x1,2.5+\y1) {Linear};
  \node [right, scale=0.75] at (7.1+\x1,2.2+\y1) {Threshold};
  \node [every neuron/.try, neuron 0/.try, scale=0.8 ] (Legend1) at (7+\x1,2.5+\y1) {\includegraphics[width=0.33cm]{figures/linear.eps}};

  \node [every neuron/.try, neuron 0/.try, scale=0.8 ] (Legend2) at (7+\x1,2.2+\y1) {\includegraphics[width=0.39cm]{figures/threshold.eps}};
\end{tikzpicture}
  \caption{The $\HAC$ concept class as a neural network where $d=4$, $r=2$ and $k=3$. The Boolean function $h$ is realized as a neural network $h_\theta$.}%
  \label{fig: HAC}
\end{figure}

%% file: figures/table_OCE.tex
\begin{longtable}{lrrrr}
\toprule
                        DSName &  HANN15 &  HANN100 &    SNN &   DENN \\
\midrule
\endfirsthead

\toprule
                        DSName &  HANN15 &  HANN100 &    SNN &   DENN \\
\midrule
\endhead
\midrule
\multicolumn{5}{r}{{Continued on next page}} \\
\midrule
\endfoot

\bottomrule
\endlastfoot
                       abalone &   63.41 &    65.13 &  66.57 &  66.38 \\
            acute-inflammation &  100.00 &   100.00 & 100.00 & 100.00 \\
               acute-nephritis &  100.00 &   100.00 & 100.00 & 100.00 \\
                         adult &   84.32 &    85.04 &  84.76 &  84.80 \\
                     annealing &   47.00 &    74.00 &  76.00 &  75.00 \\
                    arrhythmia &   62.83 &    64.60 &  65.49 &  67.26 \\
                 audiology-std &   56.00 &    68.00 &  80.00 &  76.00 \\
                 balance-scale &   92.95 &    96.79 &  92.31 &  98.08 \\
                      balloons &  100.00 &   100.00 & 100.00 & 100.00 \\
                          bank &   88.50 &    88.05 &  89.03 &  89.65 \\
                         blood &   75.94 &    75.40 &  77.01 &  73.26 \\
                 breast-cancer &   70.42 &    63.38 &  71.83 &  69.01 \\
            breast-cancer-wisc &   97.71 &    98.29 &  97.14 &  97.71 \\
       breast-cancer-wisc-diag &   97.89 &    98.59 &  97.89 &  98.59 \\
       breast-cancer-wisc-prog &   73.47 &    71.43 &  67.35 &  71.43 \\
                 breast-tissue &   61.54 &    80.77 &  73.08 &  65.38 \\
                           car &   98.84 &   100.00 &  98.38 &  98.84 \\
     cardiotocography-10clases &   78.91 &    82.11 &  83.99 &  82.30 \\
      cardiotocography-3clases &   90.58 &    93.97 &  91.53 &  94.35 \\
                    chess-krvk &   47.75 &    72.77 &  88.05 &  80.41 \\
                   chess-krvkp &   98.62 &    99.37 &  99.12 &  99.62 \\
          congressional-voting &   61.47 &    57.80 &  61.47 &  57.80 \\
  conn-bench-sonar-mines-rocks &   78.85 &    84.62 &  78.85 &  82.69 \\
    conn-bench-vowel-deterding &   89.39 &    98.92 &  99.57 &  99.35 \\
                     connect-4 &   78.96 &    86.39 &  88.07 &  86.46 \\
                       contrac &   52.72 &    49.73 &  51.90 &  54.89 \\
               credit-approval &   81.98 &    79.65 &  84.30 &  82.56 \\
                cylinder-bands &   69.53 &    73.44 &  72.66 &  78.12 \\
                   dermatology &   98.90 &    97.80 &  92.31 &  97.80 \\
                echocardiogram &   84.85 &    87.88 &  81.82 &  87.88 \\
                         ecoli &   86.90 &    84.52 &  89.29 &  85.71 \\
                     energy-y1 &   93.23 &    97.40 &  95.83 &  95.83 \\
                     energy-y2 &   89.06 &    91.15 &  90.63 &  90.62 \\
                     fertility &   92.00 &    92.00 &  92.00 &  88.00 \\
                         flags &   39.58 &    50.00 &  45.83 &  52.08 \\
                         glass &   77.36 &    60.38 &  73.58 &  60.38 \\
             haberman-survival &   72.37 &    65.79 &  73.68 &  65.79 \\
                    hayes-roth &   71.43 &    82.14 &  67.86 &  85.71 \\
               heart-cleveland &   53.95 &    59.21 &  61.84 &  57.89 \\
               heart-hungarian &   72.60 &    79.45 &  79.45 &  78.08 \\
             heart-switzerland &   45.16 &    51.61 &  35.48 &  48.39 \\
                      heart-va &   36.00 &    30.00 &  36.00 &  32.00 \\
                     hepatitis &   82.05 &    82.05 &  76.92 &  79.49 \\
                   hill-valley &   66.83 &    68.81 &  52.48 &  54.62 \\
                   horse-colic &   80.88 &    83.82 &  80.88 &  82.35 \\
             ilpd-indian-liver &   70.55 &    69.18 &  69.86 &  71.92 \\
            image-segmentation &   87.76 &    90.57 &  91.14 &  90.57 \\
                    ionosphere &   89.77 &    87.50 &  88.64 &  96.59 \\
                          iris &  100.00 &    97.30 &  97.30 & 100.00 \\
                   led-display &   73.60 &    75.20 &  76.40 &  76.00 \\
                        lenses &   50.00 &    66.67 &  66.67 &  66.67 \\
                        letter &   81.82 &    96.86 &  97.26 &  96.20 \\
                        libras &   64.44 &    81.11 &  78.89 &  77.78 \\
                 low-res-spect &   86.47 &    90.23 &  85.71 &  90.23 \\
                   lung-cancer &   37.50 &    62.50 &  62.50 &  62.50 \\
                  lymphography &   89.19 &    94.59 &  91.89 &  94.59 \\
                         magic &   86.52 &    87.49 &  86.92 &  86.81 \\
                  mammographic &   81.25 &    80.00 &  82.50 &  80.83 \\
           molec-biol-promoter &   73.08 &    80.77 &  84.62 &  88.46 \\
             molec-biol-splice &   79.05 &    78.04 &  90.09 &  85.45 \\
                       monks-1 &   65.97 &    69.91 &  75.23 &  81.71 \\
                       monks-2 &   66.20 &    66.44 &  59.26 &  65.05 \\
                       monks-3 &   54.63 &    61.81 &  60.42 &  80.09 \\
                      mushroom &  100.00 &   100.00 & 100.00 & 100.00 \\
                        musk-1 &   77.31 &    84.87 &  87.39 &  89.92 \\
                        musk-2 &   97.21 &    98.61 &  98.91 &  99.27 \\
                       nursery &   99.75 &    99.91 &  99.78 & 100.00 \\
 oocytes-merluccius-nucleus-4d &   86.27 &    83.14 &  82.35 &  83.92 \\
  oocytes-merluccius-states-2f &   92.16 &    92.55 &  95.29 &  92.94 \\
oocytes-trisopterus-nucleus-2f &   81.14 &    82.02 &  79.82 &  82.46 \\
 oocytes-trisopterus-states-5b &   93.86 &    96.05 &  93.42 &  94.74 \\
                       optical &   93.10 &    95.94 &  97.11 &  96.38 \\
                         ozone &   96.53 &    95.58 &  97.00 &  97.48 \\
                   page-blocks &   96.49 &    96.13 &  95.83 &  96.13 \\
                    parkinsons &   87.76 &    89.80 &  89.80 &  85.71 \\
                     pendigits &   94.40 &    97.11 &  97.06 &  97.37 \\
                          pima &   71.88 &    73.44 &  75.52 &  69.79 \\
    pittsburg-bridges-MATERIAL &   88.46 &    92.31 &  88.46 &  92.31 \\
       pittsburg-bridges-REL-L &   76.92 &    73.08 &  69.23 &  73.08 \\
        pittsburg-bridges-SPAN &   60.87 &    69.57 &  69.57 &  73.91 \\
      pittsburg-bridges-T-OR-D &   84.00 &    84.00 &  84.00 &  84.00 \\
        pittsburg-bridges-TYPE &   65.38 &    65.38 &  65.38 &  57.69 \\
                      planning &   66.67 &    55.56 &  68.89 &  60.00 \\
                  plant-margin &   50.50 &    79.50 &  81.25 &  83.25 \\
                   plant-shape &   39.00 &    66.50 &  72.75 &  72.50 \\
                 plant-texture &   51.75 &    75.25 &  81.25 &  81.00 \\
                post-operative &   40.91 &    63.64 &  72.73 &  68.18 \\
                 primary-tumor &   54.88 &    47.56 &  52.44 &  53.66 \\
                      ringnorm &   90.43 &    85.35 &  97.51 &  97.57 \\
                         seeds &   92.31 &    96.15 &  88.46 &  92.31 \\
                       semeion &   74.37 &    92.71 &  91.96 &  96.73 \\
                       soybean &   77.93 &    88.83 &  85.11 &  88.03 \\
                      spambase &   93.57 &    94.17 &  94.09 &  94.87 \\
                         spect &   62.90 &    63.44 &  63.98 &  62.37 \\
                        spectf &   91.98 &    91.98 &  49.73 &  89.30 \\
     statlog-australian-credit &   65.12 &    63.37 &  59.88 &  61.05 \\
         statlog-german-credit &   72.40 &    72.40 &  75.60 &  72.00 \\
                 statlog-heart &   85.07 &    91.04 &  92.54 &  92.54 \\
                 statlog-image &   95.15 &    96.88 &  95.49 &  97.75 \\
               statlog-landsat &   87.55 &    89.25 &  91.00 &  89.90 \\
               statlog-shuttle &   99.92 &    99.92 &  99.90 &  99.91 \\
               statlog-vehicle &   78.67 &    77.25 &  80.09 &  81.04 \\
                  steel-plates &   73.61 &    76.49 &  78.35 &  77.53 \\
             synthetic-control &   94.00 &    98.00 &  98.67 &  99.33 \\
                      teaching &   57.89 &    57.89 &  50.00 &  57.89 \\
                       thyroid &   98.37 &    98.25 &  98.16 &  98.22 \\
                   tic-tac-toe &   96.65 &    97.07 &  96.65 &  98.33 \\
                       titanic &   78.73 &    78.73 &  78.36 &  78.73 \\
                        trains &  100.00 &    50.00 &    NaN &    NaN \\
                       twonorm &   97.30 &    98.27 &  98.05 &  98.16 \\
      vertebral-column-2clases &   88.31 &    85.71 &  83.12 &  85.71 \\
      vertebral-column-3clases &   81.82 &    80.52 &  83.12 &  80.52 \\
                wall-following &   92.45 &    94.79 &  90.98 &  91.86 \\
                      waveform &   85.84 &    84.00 &  84.80 &  83.92 \\
                waveform-noise &   84.72 &    84.96 &  86.08 &  84.32 \\
                          wine &   97.73 &   100.00 &  97.73 & 100.00 \\
              wine-quality-red &   62.50 &    65.00 &  63.00 &  63.50 \\
            wine-quality-white &   54.82 &    61.03 &  63.73 &  62.25 \\
                         yeast &   59.03 &    60.65 &  63.07 &  58.22 \\
                           zoo &   96.00 &    96.00 &  92.00 & 100.00 \\
\end{longtable}